\documentclass{article}
\usepackage{microtype}
\usepackage{graphicx}
\usepackage{subcaption}
\usepackage{booktabs} 

\usepackage{hyperref}

\usepackage[preprint]{icml2026}

\usepackage{amsmath}
\usepackage{amssymb}
\usepackage{mathtools}
\usepackage{amsthm}
\usepackage[shortlabels]{enumitem}

\usepackage{color}

\usepackage[capitalize,noabbrev]{cleveref}

\usepackage[framemethod=TikZ]{mdframed}
\mdfdefinestyle{graybox}{%
    linecolor=white,
    outerlinewidth=0pt,
    roundcorner=3pt,
    innertopmargin=8pt,
    innerbottommargin=8pt,
    innerrightmargin=10pt,
    innerleftmargin=10pt,
    backgroundcolor=black!5!white}

\theoremstyle{plain}
\newtheorem{theoreminner}{Theorem}[section]
\newtheorem{propositioninner}[theoreminner]{Proposition}
\newtheorem{lemmainner}[theoreminner]{Lemma}
\newtheorem{corollaryinner}[theoreminner]{Corollary}

\newenvironment{theorem}[1][]
  {\begin{mdframed}[style=graybox]\begin{theoreminner}[#1]}
  {\end{theoreminner}\end{mdframed}}

\newenvironment{lemma}[1][]
  {\begin{mdframed}[style=graybox]\begin{lemmainner}[#1]}
  {\end{lemmainner}\end{mdframed}}

\theoremstyle{definition}
\newtheorem{definitioninner}[theoreminner]{Definition}
\newtheorem{assumptioninner}[theoreminner]{Assumption}
\newenvironment{definition}[1][]
  {\begin{mdframed}[style=graybox]\begin{definitioninner}[#1]}
  {\end{definitioninner}\end{mdframed}}
\newenvironment{assumption}[1][]
  {\begin{mdframed}[style=graybox]\begin{assumptioninner}[#1]}
  {\end{assumptioninner}\end{mdframed}}

\theoremstyle{remark}
\newtheorem{remark}[theoreminner]{Remark}

\icmltitlerunning{Stochastic Gradient Lattice Random Walk}

\begin{document}

\twocolumn[
  \icmltitle{Robust Stochastic Gradient Posterior Sampling with Lattice Based Discretisation}
  \icmlsetsymbol{equal}{*}

  \begin{icmlauthorlist}
    \icmlauthor{Zier Mensch}{UVAPhys,NTU}
    \icmlauthor{Lars Holdijk}{Ox,NC}
    \icmlauthor{Samuel Duffield}{NC} \\
    \icmlauthor{Maxwell Aifer}{NC}
    \icmlauthor{Patrick J. Coles}{NC}
    \icmlauthor{Max Welling}{AMLab}
    \icmlauthor{Miranda Cheng}{UVAPhys,Math,Kort}
  \end{icmlauthorlist}

  \icmlaffiliation{UVAPhys}{Institute of Physics, University of Amsterdam, Netherlands}
  \icmlaffiliation{Math}{Institute for Mathematics, Academia Sinica, Taiwan}
  \icmlaffiliation{Kort}{Korteweg-de Vries Institute for Mathematics, University of Amsterdam, Netherlands}
  \icmlaffiliation{Ox}{Department of Computer Science, University of Oxford, United Kingdom}
  \icmlaffiliation{NC}{Normal Computing Corporation, New York, New York, USA}
  \icmlaffiliation{AMLab}{Amsterdam Machine Learning Lab, University of Amsterdam, Netherlands}
  \icmlaffiliation{NTU}{Department of Physics, National Taiwan University, Taiwan}

  \icmlcorrespondingauthor{Zier Mensch}{ziermensch@gmail.com}
  % \icmlcorrespondingauthor{Lars Holdijk}{larsholdijk@gmail.com}

  \icmlkeywords{Machine Learning, ICML}

  \vskip 0.3in
]



\printAffiliationsAndNotice{}  % no special notice (required even if empty)

\begin{abstract}
Stochastic-gradient MCMC methods enable scalable Bayesian posterior sampling but often suffer from sensitivity to minibatch size and gradient noise. To address this, we propose Stochastic Gradient Lattice Random Walk (SGLRW), an extension of the Lattice Random Walk discretization. Unlike conventional Stochastic Gradient Langevin Dynamics (SGLD), SGLRW introduces stochastic noise only through the off-diagonal elements of the update covariance; this yields greater robustness to minibatch size while retaining asymptotic correctness. Furthermore, as comparison we analyze a natural analogue of SGLD utilizing gradient clipping. Experimental validation on Bayesian regression and classification demonstrates that SGLRW remains stable in regimes where SGLD fails, including in the presence of heavy-tailed gradient noise, and matches or improves predictive performance.
\end{abstract}
\section{Introduction}
Bayesian methods provide a principled framework for learning probabilistic models from data and natively capturing uncertainty by replacing the parameter point estimates in frequentist methods with a posterior distribution over parameters. By marginalizing over parameters, Bayesian methods act as a form of regularization and enable uncertainty quantification and robust model selection \citep{neal2012bayesian}. In doing so, Bayesian models can potentially mitigate overfitting and miscalibration, which are prevalent in modern large-scale, overparameterized neural networks \citep{guo2017calibration, yang2023bayesian}. Realizing these benefits in the modern hyperscaling era, however, requires posterior inference algorithms that scale to both dataset size and model complexity.

Within Bayesian methods, Markov chain Monte Carlo (MCMC) \citep{neal1993probabilistic, robert1999monte} remains the gold standard for posterior sampling, but it is also among the methods most affected by scalability and computational cost \citep{gelman1997weak}. Alternative approaches, including variational inference \citep{blei2017variational}, Laplace approximations \citep{tierney1986accurate}, and single-pass methods \citep{gal2016dropout}, are often less computationally demanding, but still introduce substantial overhead in training and inference \citep{blei2017variational, lakshminarayanan2017simple, wilson2020bayesian}. As a result, these methods have, in some settings, fallen out of favour relative to modern approaches for assessing model trustworthiness, such as explainable and interpretable models \citep{li2023trustworthy}.

\begin{figure}[t]
    \begin{center}
      \centerline{\includegraphics[width=0.85\columnwidth]{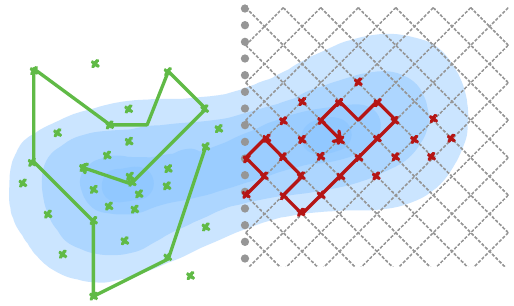}}
      \caption{
        Comparing SGLD \textbf{(left)} and SGLRW \textbf{(right)} discretisations of Langevin dynamics, we can observe that the lattice based discretisation suppresses large parameter jumps that occur due to minibatch noise, resulting in more stable sampling.
      }
      \label{fig:sketch}
    \end{center}
    \vspace{-0.5in}
  \end{figure}

One core issue of MCMC methods for Bayesian posterior inference is the theoretical requirement to evaluate the gradient of the posterior over the entire dataset at each iteration \citep{welling2011bayesian, ma2015complete}. With growing model complexity and dataset size, this is often prohibitively expensive. Stochastic-gradient variants of MCMC methods, such as Stochastic Gradient Langevin Dynamics (SGLD) \citep{welling2011bayesian}, alleviate this concern to some extent by allowing the gradient to be evaluated on a small minibatch of data at each iteration. However, these methods are still known to be sensitive to the minibatch size \citep{brosse2018promises}. As a result, they do not scale to regimes where only a small minibatch is available at each evaluation step, or where only a small number of samples from the dataset can be stored in memory, as is becoming increasingly common.

In this work, we propose Stochastic Gradient Lattice Random Walk (SGLRW), a stochastic-gradient extension of the recently introduced Lattice Random Walk (LRW) \citep{duffield2025lattice} discretisation of overdamped Langevin dynamics. LRW replaces the Gaussian increments of Langevin dynamics with bounded binary or ternary updates on a lattice. As we show, unlike SGLD, the stochastic gradient noise in SGLRW enters only through the off-diagonal elements of the covariance matrix of the update and therefore remains robust to the minibatch size. This allows SGLRW to sample from the posterior distribution with the same asymptotic correctness as SGLD, but with improved stability for small minibatches, as shown in Figure \ref{fig:sketch}.

In short our contributions are as follows:\vspace{-0.3cm}
\begin{itemize}[leftmargin=13pt, labelsep=.6em]
  \setlength\itemsep{0em}
    \item We propose Stochastic Gradient Lattice Random Walk (SGLRW), a lattice based stochastic-gradient discretisation of overdamped Langevin dynamics (Section~\ref{subsec:sglrw}).
    \item Extending the analysis of \citet{chen2015convergence}, we provide a mean-squared-error analysis that justifies its improved stability for small minibatches (Section~\ref{sec:analysis}).
    \item We validate our theoretical findings on a mix of analytically understood problems and real-world tasks, including sentiment classification using an LLM (Section~\ref{sec:experiments}).
    \item We discuss a clipped version of SGLD as a strong baseline that is analogous to gradient clipping in stochastic gradient descent (Section~\ref{subsec:Clipped_SGLD}).
 
\end{itemize}

\section{Background}
As stated in the introduction, we consider the problem of minibatch-induced instability in stochastic gradient MCMC methods for Bayesian posterior sampling. Here, we recap the necessary background on Bayesian machine learning, posterior inference, and stochastic gradient methods.

\paragraph{Bayesian Machine Learning}
We consider the supervised learning setting where we have observed data $\mathcal{D} = \{(x_i, y_i)\}_{i=1}^N$ and aim to infer a posterior distribution $p(\theta \mid \mathcal{D})$ over the parameter vector $\theta \in \mathbb{R}^d$. In contrast to frequentist approaches, which seek a single point estimate $\theta^*$ that maximises the likelihood $p(\mathcal{D} \mid \theta)$, Bayesian machine learning maintains a full distribution over parameters, $p(\theta \mid \mathcal{D})$. This \textit{posterior distribution} captures the uncertainty in our parameter estimates given the observed data. 

The posterior distribution is given by Bayes' theorem as
\begin{align}
p(\theta \mid \mathcal{D}) = \frac{p(\mathcal{D} \mid \theta) p(\theta)}{p(\mathcal{D})} \propto p(\theta) \prod_{i=1}^N p(y_i \mid x_i, \theta),
\end{align}
where $p(\mathcal{D} \mid \theta)$ is again the \textit{likelihood} and $p(\theta)$ is the \textit{prior}. Notably, we will often write the posterior distribution as $p(\theta \mid \mathcal{D}) \propto \exp[-U(\theta)]$ where 
\begin{align}
U(\theta) = -\log p(\theta) - \sum_{i=1}^N \log p(y_i \mid x_i, \theta),
\end{align}
is the negative log-posterior.

Given the posterior distribution, predictions for a new input $x^*$ are then obtained via the posterior predictive distribution,
\begin{align} \label{eqn:posterior_predictive}
p(y^* \mid x^*, \mathcal{D}) = \int p(y^* \mid x^*, \theta) \, p(\theta \mid \mathcal{D}) \, d\theta,
\end{align}
which marginalises over the parameters $\theta$ and can thus provide calibrated predictive distributions. Crucially, however, evaluating \eqref{eqn:posterior_predictive} is generally infeasible, as it involves a high-dimensional integral over $\theta$.

\subsection{Bayesian Posterior Sampling}
A common approach to Bayesian inference is to replace the integral in \eqref{eqn:posterior_predictive} with a Monte Carlo average. Most commonly, this is achieved by drawing parameter samples $\theta \sim p(\theta \mid \mathcal{D})$ from the posterior using a Markov chain Monte Carlo (MCMC) approach.

In this work, we specifically focus on MCMC samplers that are expressed as discretisations of stochastic differential equations (SDEs) whose stationary distribution coincides with the target posterior \citep{ma2015complete}. Among these, the most common choice is the overdamped Langevin diffusion,
\begin{align} \label{eq:overdamped_general_sde}
d\theta_t &= f(\theta_t)\,dt + \sqrt{2D(\theta_t)}\,dW_t ,
\end{align}
where $D(\theta_t)$ is a symmetric positive semidefinite diffusion matrix and $f(\theta_t)$ is the drift. In the context of Bayesian posterior sampling, the drift is given by $f(\theta_t) = - \nabla U(\theta_t)$ and the diffusion matrix is typically set to $D(\theta_t) = I$, resulting in the following SDE:
\begin{align}
d\theta_t &= - \nabla U(\theta_t) \, dt + \sqrt{2} \, dW_t \label{eq:overdamped_U_sde} \\
&= (\nabla \log p(\theta_t) + \sum_{i=1}^N \nabla \log p(y_i \mid x_i, \theta_t)) dt + \sqrt{2} \, dW_t.
\end{align}

\subsubsection{Stochastic Gradient Methods.}
As discussed in the introduction, in large-scale settings the requirement to evaluate the gradient of the posterior over the entire dataset at each iteration is limiting. As such, stochastic gradient MCMC (SG-MCMC) methods replace the full posterior gradient with an unbiased minibatch estimator:
\begin{align}
  \widehat{\nabla U}(\theta;\mathcal{B}) = -\nabla \log p(\theta)
    - \frac{N}{B}\sum_{i\in\mathcal{B}}\nabla \log p(y_i\mid x_i,\theta),
  \end{align}
  where $\mathcal{B} \subset \{1,\ldots,N\}$ is a minibatch index set of size $B$, and $\{(x_i,y_i)\}_{i\in\mathcal{B}} \subset \mathcal{D}$ are the corresponding datapoints.

  \paragraph{Stochastic Gradient Langevin Dynamics} Stochastic Gradient methods for Bayesian posterior sampling were first introduced by \citet{welling2011bayesian} for Langevin dynamics, and only later generalised to other MCMC samplers \citep{chen2014stochastic, ma2015complete}. Concretely, \citet{welling2011bayesian} obtain the following Stochastic Gradient Langevin Dynamics (SGLD) update rule:

\vspace{0.3cm}
\begin{definition}[SGLD Update Rule]
    Given a minibatch $\mathcal{B}$, applying the Euler-Maruyama discretisation to the Langevin SDE~\Cref{eq:overdamped_U_sde} and replacing the full gradient with a minibatch estimate yields the SGLD update:
    \begin{align}
    \theta_{t+1} = \theta_t - \delta_t\,\widehat{\nabla U}(\theta_t;\mathcal{B}) + \sqrt{2\delta_t}\,\xi_t,
    \end{align}
    where $\xi_t\sim\mathcal N(0,I)$.
\end{definition}
  \vspace{0.1cm}
For the original full-gradient Langevin dynamics, Euler-Maruyama has weak order $1$ convergence. However, with stochastic gradients, the convergence properties are more subtle \citep{vollmer2016exploration}. With a fixed step size $\delta_t = \delta$, SGLD converges to a stationary distribution that is biased relative to the true posterior, reflecting both discretisation error and the noise from minibatch gradients. To obtain asymptotically exact expectations, a decreasing step size schedule satisfying $\delta_t \to 0$, $\sum_t \delta_t = \infty$, and $\sum_t \delta_t^2 < \infty$ needs to be used \citep{teh2016consistency}. However, this does come at the cost of slower mixing as the step size vanishes.

In our analysis below, we decompose the stochastic gradient as $\widehat{\nabla U}(\theta;\mathcal{B})
= \nabla U(\theta) + \zeta(\theta;\mathcal{B})$ and define $G(\theta)$ to be the minibatch-induced gradient covariance, $\mathrm{Cov}_{\mathcal B}[\zeta(\theta;{\mathcal B})\mid\theta]=G(\theta)$. Clearly we have $\mathbb{E}_{\mathcal B}[\zeta(\theta; \mathcal B)\mid\theta]=0$.
\section{Related Work}
We now briefly review related work on batch-size sensitivity in SG-MCMC methods and large-scale Bayesian inference.

\paragraph{Batch-Size Sensitivity.}
As stated, in practice, SGLD often requires large minibatches for stability, limiting scalability \citep{baker2019control}. Variance-reduction and control-variate methods reduce noise but introduce memory overhead or require periodic full-data passes \citep{dubey2016variance, baker2019control, li2020improving}. Alternative approaches include adaptive subsampling \citep{korattikara2014austerity}, preconditioning \citep{li2016preconditioned}, and importance sampling \citep{li2020improving}. Moreover, recent work characterizes stochastic gradient noise as heavy-tailed rather than Gaussian \citep{simsekli2019tail} and proposes specialized fractional dynamics to retarget the sampling distribution when such noise is present \citep{simsekli2020fractional}. While these methods primarily address the statistical properties of the noise, our solution improves robustness through its lattice-based discretisation, enabling stable sampling with small minibatches.

\paragraph{Large-Scale Bayesian Inference.}
Bayesian uncertainty estimation is important for large neural models, including large language models, where calibration and robustness are critical \citep{yang2023bayesian}. Scalable approximations such as Laplace methods \citep{daxberger2021laplace, yang2023bayesian, chen2024bayesian, sliwa2025mitigating} and variational inference \citep{harrison2024variational, wang2024blob, xiang2025fine, samplawski2025scalable} trade accuracy for efficiency. Recent work shows that sampling-based methods can be applied to large models with appropriate algorithmic structure, as in SGLD-Gibbs \citep{kim2025lift}.
\section{Stochastic Gradient Lattice Random Walk}
With the background established in the previous section, we now introduce the proposed method, Stochastic Gradient Lattice Random Walk (SGLRW). For this purpose, we first review the Lattice Random Walk (LRW) discretisation of Langevin dynamics and then introduce the proposed stochastic gradient extension.

\subsection{Lattice Random Walk}
The Lattice Random Walk (LRW) scheme, recently introduced in \citep{duffield2025lattice}, proposes an alternative to standard SDE discretisations, such as Euler-Maruyama, by substituting Gaussian noise with bounded binary increments. As shown in \citep{duffield2025lattice}, this has the benefit of being stable under non-Lipschitz gradients, which are common in deep learning. Additionally, the structure of LRW also lends itself to low-precision, stochastic hardware \citep{alaghi2013survey} and thermodynamic hardware \citep{conte2019thermodynamic} which has recently started to be developed for AI applications \citep{melanson2025thermodynamic}.

In LRW, at each iteration, the parameters are updated as
\begin{align}
\Delta\theta_{t+1} = S_t,
\qquad
(S_t)_i \in \{-\sqrt{2\delta_t}, +\sqrt{2\delta_t}\},
\end{align}
where each direction $(S_t)_i$ is sampled independently from the following state-dependent probability
\begin{align}
\mathbb{P}\!\left[(S_t)_i = \pm\sqrt{2\delta_t}\,\middle|\,\theta_t\right]
= \tfrac12 \pm \tfrac12 \sqrt{\tfrac{\delta_t}{2}}\,f_i(\theta_t).
\end{align}
These probabilities are valid whenever $\sqrt{\delta_t/2}\,|f_i(\theta_t)| \le 1$, where we write $f_i(\theta)$ to denote the $i$th component of the drift vector field $f(\theta)$.

By construction, the first two conditional moments satisfy
\begin{align}
\mathbb{E}[S_t \mid \theta_t] = \delta_t f(\theta_t),
\qquad
\mathbb{E}[S_t S_t^\top \mid \theta_t] = 2\delta_t I,
\end{align}
and LRW is shown to be weakly first-order consistent with the continuous-time Langevin dynamics in \Cref{eq:overdamped_general_sde} (Theorem~1 of \citep{duffield2025lattice}). With the specific choice of $f(\theta) = - \nabla U(\theta)$, LRW thus provides a valid discretisation of the Langevin dynamics in \Cref{eq:overdamped_U_sde}.

\begin{figure}[!b]
  \centering
  \includegraphics[width=\linewidth]{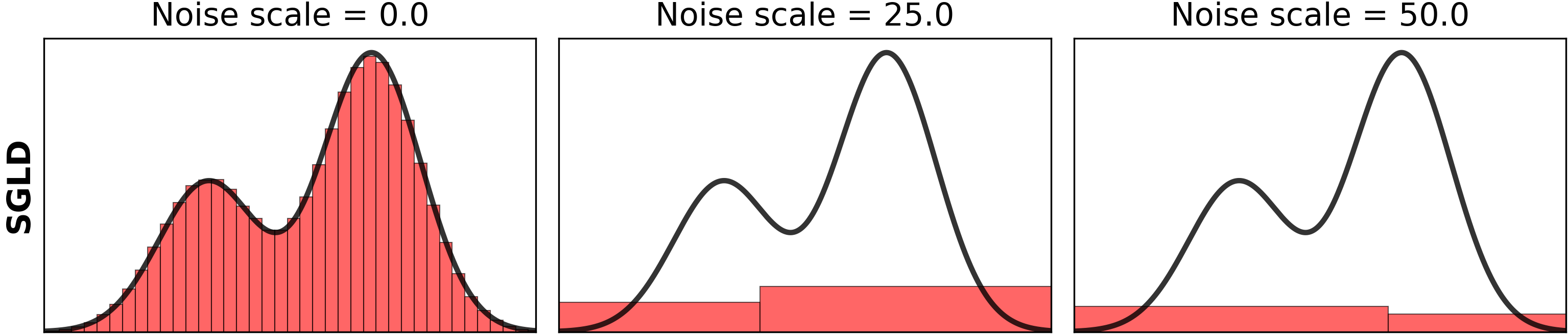}
  \includegraphics[width=\linewidth]{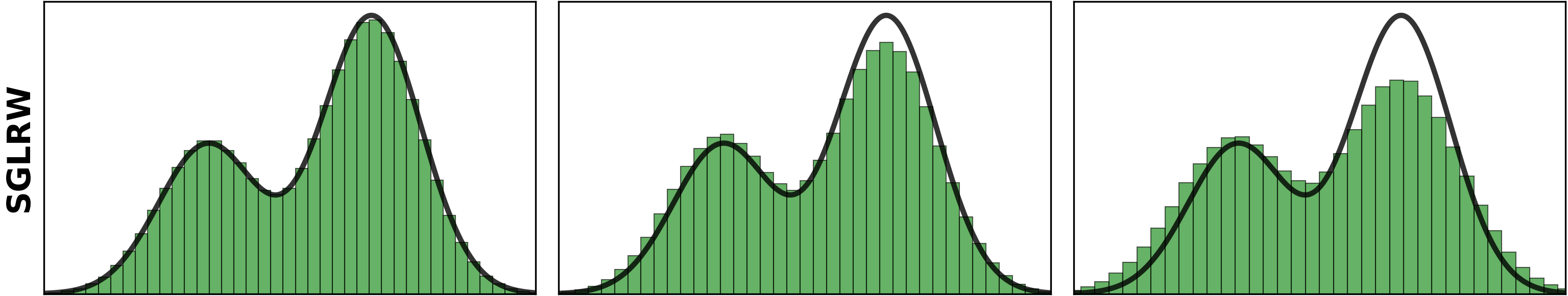}
  \caption{Multimodal univariate target with exact gradient corrupted by synthetic $\alpha$-stable noise ($\alpha = 1.5$) of increasing scale. We observe that as we increase the noise scale SGLD quickly fails while SGLRW remains stable.}
  \label{fig:1d}
  \vspace{-0.5cm}
\end{figure}

\subsection{Stochastic Gradient Lattice Random Walk} \label{subsec:sglrw}
We now come to the main contribution of our work, the proposal of Stochastic Gradient Lattice Random Walk (SGLRW), which replaces the stochastic gradient update rule of SGLD with a lattice-based update:

\vspace{0.2cm}
\begin{definition} [SGLRW Update Rule]\label{def:sglrw-update-rule}
  Given a minibatch $\mathcal{B}$, at the $t$th iteration, the Stochastic Gradient Lattice Random Walk updates the parameter vector as
  \begin{align}
    \theta_{t+1} = \theta_t + S_t.
  \end{align}
  where each coordinate $(S_t)_i\in\{-\sqrt{2\delta_t},+\sqrt{2\delta_t}\}$ is sampled from the state-dependent probability
  \begin{align}
  \mathbb{P}\!\left[(S_t)_i = \pm \sqrt{2\delta_t}\,\middle|\,\theta_t\right]  = \tfrac12 \mp \tfrac12 \sqrt{\tfrac{\delta_t}{2}}\, \widehat{\partial_i U}(\theta_t; \mathcal{B}),
  \label{eq:sglrw-transition}
  \end{align}
  which is valid whenever $\sqrt{\delta_t/2}\,|\widehat{\partial_i U}(\theta_t;\mathcal{B})| \le 1$.
\end{definition}

We hypothesize, and analytically evaluate next, that due to the bounded structure, large fluctuations in stochastic gradients have a less severe impact on the update in the case of SGLRW than in the case of SGLD. This is illustrated in Figure~\ref{fig:1d}, in a one-dimensional multimodal example.

\paragraph{Heavy-Tailed Noise} Using the stochastic gradient as $\widehat{\nabla U}(\theta;\mathcal{B})
= \nabla U(\theta) + \zeta(\theta;\mathcal{B})$, we set $U(\theta)$ to be the negative log-probability of the multimodal Gaussian, and choose $\zeta(\theta;\mathcal{B})$ to follow a heavy-tailed $\alpha$-stable distribution with $\alpha<2$, for which second moments do not exist. This distribution was shown to closely resemble the minibatch gradient noise in the standard SGD setting by \citet{simsekli2019tail}. 

As Figure~\ref{fig:1d} shows, in this regime, where stability depends critically on whether large stochastic fluctuations can induce rare but catastrophic updates, SGLD fails while SGLRW remains stable. In Appendix~\ref{app:clippedSGLDAppendix} we also provide an analysis using Gaussian gradient noise (Figure~\ref{fig:fullheavies}).

\subsection{Mean Squared Error Analysis}\label{sec:analysis}
Having introduced SGLRW, we now present an analysis of the differences between SGLD and SGLRW that highlight the benefits of using SGLRW with small batch sizes.

We follow a similar approach to the analysis of SGLD in \citet{chen2015convergence}, focussing on the mean squared error (MSE) $\mathbb{E}(\hat\phi-\bar\phi)^2$ between the true posterior expectation
\begin{align}
  \bar\phi := \int \phi(\theta)\,p(\theta \mid \mathcal{D})\,d\theta. 
\end{align}
and the ergodic average
\begin{align}
  \hat\phi := \frac{1}{L}\sum_{n=1}^L \phi(\theta_{n\delta_t}),
\end{align}
over the discrete-time Markov chain $\{\theta_{n\delta_t}\}_{n\ge 0}$ generated by an SG-MCMC method, such as SGLD or SGLRW, with step size $\delta_t$. Here, $\phi : \mathbb{R}^d \to \mathbb{R}^m$ represents a smooth test function, such as the posterior predictive distribution of a new data point $x^*$, as defined in \Cref{eqn:posterior_predictive}.

As a comparison of the MSE for SGLRW and SGLD, we present the following theorem:

\begin{figure*}[ht]
  \centering
    \includegraphics[width=1.\linewidth]{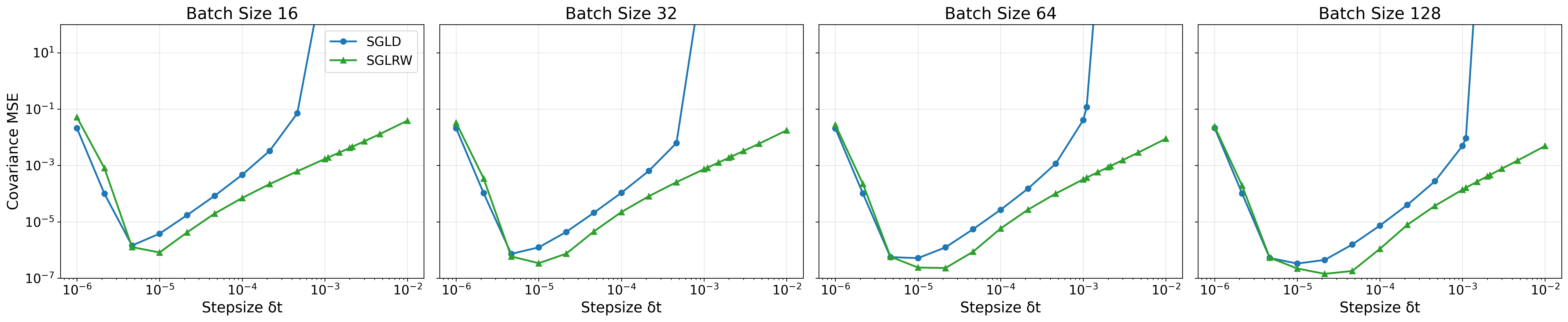}
  \caption{Mean-squared error (MSE) of the posterior covariance as a function of the step size $\delta_t$, shown for different batch sizes for $50$-dimensional Bayesian linear regression.}
  \label{fig:mseplots}
\end{figure*}

\vspace{0.3cm}
\begin{theorem}
\label{thm:refined-mse-bound}
Under the Assumption~\ref{ass:lyap}, we find that the MSE is bounded by the following three contributions
$$\text{MSE} \le C \left( \mathcal{E}_{\text{drift}} + \mathcal{E}_{\text{disc}} + \mathcal{E}_{\text{cov}} \right)$$
for some $C$ that depends on the target distribution. The covariance error term for SGLRW is never larger than that of SGLD
$$\mathcal{E}_{\text{cov}}^{\text{SGLRW}} \leq \mathcal{E}_{\text{cov}}^{\text{SGLD}}$$
while the other contributions are the same for both. 
Moreover, it is strictly smaller whenever $2\partial_iU\zeta_i +\zeta_i^2$ is non-vanishing for some direction $i$.
\end{theorem}

The first statement regarding the MSE upper bound, and the precise expression for each of the three contributions, is the content of Theorem \ref{thm:app_refined-mse}, which is an extension of Theorem~3 of \citet{chen2015convergence} to take into account non-vanishing second-order contributions.
In particular, the covariance contribution is given by 
\begin{align}
\mathcal{E}_{\text{cov}}= \frac{\delta_t^2}{L}\sum_{n=1}^L \mathbb{E}\big[\|M_n\|_F^2 \big].
\end{align}
In the above, 
the scheme-dependent second-order error $M_n$ induced by minibatching is defined as
\begin{align}\label{eq:defAn}
  M_n(\theta;\mathcal B_n)
  &:= \delta_t^{-2}\,\mathbb E_{\varepsilon_n}\!\Big[
  \Delta\theta_n^{\mathrm{fb}}(\Delta\theta_n^{\mathrm{fb}})^\top
  -\Delta\theta_n^{\mathrm{mb}}(\Delta\theta_n^{\mathrm{mb}})^\top
  \,\Big| \notag\\
  &\hspace{2.2em} \theta_{(n-1)\delta_t} = \theta,\mathcal B_n
  \Big],
  \end{align}
where $\Delta\theta_n^{\mathrm{mb}}$ and $\Delta\theta_n^{\mathrm{fb}}$ denote the one-step increments of the minibatch and full-batch updates, respectively. 

To prove the bound in~\Cref{thm:refined-mse-bound}, we observe a lemma quantifying the difference between the second-order structure of SGLD and SGLRW 
\vspace{0.4cm}
\begin{lemma} \label{lem:mse-bound}
  The second moment error of the minibatch update for SGLRW satisfies
  \begin{equation}
    M_{n,\mathrm{SGLRW}}(\theta;\mathcal B_n) =\mathrm{offdiag}\!\big(M_{n,\mathrm{SGLD}}(\theta;\mathcal B_n)\big),
  \end{equation}
  where $M_{n,\mathrm{SGLD}}(\theta,\mathcal B_n)$ is the second-order error for SGLD given by
  $$
M_{n,\mathrm{SGLD}}
=
\zeta\zeta^\top
+\nabla U \zeta^\top
+\zeta \nabla U^\top.
  $$
\end{lemma}
\begin{proof}
   See Appendix~\ref{app:covariance-derivation}.
\end{proof} 
The above lemma highlights the fact that, in SGLRW, the lattice constraint enforces fixed-magnitude coordinate updates, so the diagonal of the one-step second moment of the increment is deterministic. In contrast, for SGLD this diagonal depends on the stochastic gradient and is inflated by minibatch noise.

Combining the above, we readily obtain the error bound of Theorem \ref{thm:refined-mse-bound}. 
This shows that  under some mild conditions, the SGLRW discretisation achieves a strictly tighter MSE bound than the SGLD, leading to a more robust implementation of minibatch gradient updates. 

\subsubsection{Validation and practical considerations} 
We experimentally validate this theoretical finding in Figure~\ref{fig:mseplots}, where we compare the MSE for SGLRW and SGLD for Bayesian linear and logistic regression. As $\delta_t$ increases, SGLD becomes unstable and the covariance error explodes, while SGLRW remains stable for all batch sizes. A further discussion of the experimental setup used to generate this insight is provided in~\Cref{subsec:EXP_Bayesian_Linear_Regression}. 

Despite the theoretical guarantee in  Lemma~\ref{lem:mse-bound},  
we note that the theoretical restriction $\sqrt{\delta_t/2}\,|\widehat{\partial_i U}(\theta_t;\mathcal{B})| \le 1$ can be challenging to arrange in practice without overly conservative step size tuning. 
To address this, in the implementation we clip the quantity $\sqrt{\delta_t/2}\,|\widehat{\partial_i U}(\theta_t;\mathcal{B})|$ to one, instead of step size tuning. Although this clipping can technically lead to a bias, our empirical evaluations suggest that this does not pose any serious problem in practice, in the regimes of $\delta_t$ and minibatch sizes that one is interested in.     
\section{Experimental evaluation}\label{sec:experiments}
We evaluate SGLRW on posterior sampling problems for linear regression, logistic regression, and predictive classification tasks. Across all experiments, we vary the minibatch size $B$ and base step size $\delta_t$ under matched decaying learning-rate schedules of the form $\delta_t(1+t)^{-0.55}$, in line with \citep{welling2011bayesian}. All experiments were implemented in \texttt{posteriors} \citep{duffield2024scalable}.

\subsection{Strong Baseline: Clipped SGLD.}\label{subsec:Clipped_SGLD}
Similar to our analysis in~\Cref{sec:analysis}, we also compare SGLRW against SGLD in our empirical evaluation here. Additionally, we introduce \textit{Clipped-SGLD} as an additional strong baseline. Gradient clipping is standard in large-scale SGD, where saturating the drift prevents rare large gradients from producing unstable updates. It is therefore natural to ask whether the same stabilisation can be applied to SGLD.

We define the Clipped-SGLD update rule as follows:
\begin{align}\label{subsec:Clipped_SGLD_baseline}
    \theta_{t+1} = \theta_t
      -\mathrm{clip}(\delta_t\,\widehat{\nabla U}(\theta_t;\mathcal{B});R)
      + \sqrt{2\delta_t}\,\xi_t, 
\end{align}
where $\mathrm{clip}(x;R)_i = \mathrm{sign}(x_i)\min\{|x_i|,R\}$ and $R=\sqrt{2\delta_t}$. Note that this is a componentwise clipping operation, and deviates slightly from the standard definition of gradient clipping in SGD. In gradient clipping for standard SGD, the clipping is performed over the entire update vector $\Delta\theta_t$, while here we clip only part of the update vector. Clipping the entire update vector would result in the SDE having a different stationary distribution (Appendix~\ref{app:full-increment-clipping-limit}), while drift-truncated Euler schemes are known to converge to the exact Langevin dynamics in the small-step limit \citep{roberts1996exponential, hutzenthaler2015numerical}. 

In Appendix~\ref{app:clippedSGLDAppendix} we provide the same MSE and heavy-tailed noise analysis for Clipped-SGLD as previously discussed in~\Cref{sec:analysis} for SGLRW.

\begin{table}[t]
    \caption{Kullback--Leibler (KL) divergence between the true posterior and the empirical Gaussian fit of the samples, shown for different samplers (SGLD, SGLRW, Clipped SGLD), minibatch size $B$, and base learning rate $\delta_t$. 
    The Monte Carlo reference KL divergence, which quantifies the intrinsic sampling variability when estimating the analytic posterior, is $0.055201$.
    \textbf{Bold} indicates the lowest KL divergence for a given $(B, \delta_t)$ pair.}
    \centering
    \label{tab:kl_comparison_bold}
    \setlength{\tabcolsep}{4pt}
    \begin{small}
    \begin{tabular}{r c c c c}
    \hline
    \multicolumn{2}{c}{\textbf{Hyperparameters}} & \multicolumn{3}{c}{\textbf{KL Divergence}} \\
    \cline{1-2} \cline{3-5}
    $B$ & $\delta_t$ & \textbf{SGLD} & \textbf{SGLRW} & \textbf{Clipped SGLD} \\
    \hline
    8 & $10^{-3}$ & 19.889 & \textbf{6.060} & 18.184 \\
    8 & $10^{-4}$ & 0.483 & \textbf{0.202} & 0.777 \\
    \hline
    16 & $10^{-3}$ & 7.155 & \textbf{2.317} & 8.530 \\
    16 & $10^{-4}$ & 0.175 & \textbf{0.070} & 0.204 \\
    \hline
    32 & $10^{-3}$ & 2.441 & \textbf{0.729} & 3.540 \\
    32 & $10^{-4}$ & 0.087 & \textbf{0.064} & 0.091 \\
    \hline
    64 & $10^{-3}$ & 0.838 & \textbf{0.165} & 1.114 \\
    64 & $10^{-4}$ & 0.063 & \textbf{0.055} & 0.067 \\
    \hline
    128 & $10^{-3}$ & 0.315 & \textbf{0.074} & 0.351 \\
    128 & $10^{-4}$ & \textbf{0.054} & \textbf{0.054} & 0.061 \\
    \hline
    256 & $10^{-3}$ & 0.140 & \textbf{0.065} & 0.141 \\
    256 & $10^{-4}$ & \textbf{0.051} & 0.056 & 0.062 \\
    \hline
    512 & $10^{-3}$ & 0.086 & \textbf{0.065} & 0.088 \\
    512 & $10^{-4}$ & \textbf{0.052} & 0.054 & 0.062 \\
    \hline
    1000 & $10^{-3}$ & \textbf{0.058} & \textbf{0.058} & 0.060 \\
    1000 & $10^{-4}$ & \textbf{0.052} & 0.054 & 0.061 \\
    \hline
    \end{tabular}
    \end{small}
    \end{table} 

\subsection{Bayesian Linear Regression}\label{subsec:EXP_Bayesian_Linear_Regression}
We first evaluate SGLRW using a linear--Gaussian model where the posterior admits a closed form. Using the closed-form solution, we can analytically compute the KL divergence between the true posterior and the empirical Gaussian fit to the samples. This allows us to provide a more rigorous evaluation of the empirical performance of SGLRW compared to SGLD and Clipped-SGLD. 

Concretely, the linear model we consider is given by
\begin{align}
    y = X\theta + \varepsilon, \qquad \varepsilon \sim \mathcal{N}(0,\sigma^2 I),
\end{align}
where $X\in\mathbb{R}^{N\times d}$ is the design matrix and $\varepsilon\in\mathbb{R}^N$ is the noise vector. With a Gaussian prior $p(\theta) = \mathcal{N}(0,\tau^{-1} I)$, the resulting posterior $\mathcal{N}(\mu,\Sigma)$ is therefore given by
\begin{align}
    \Sigma^{-1} = \tfrac{1}{\sigma^2} X^\top X + \tau I,
    \qquad
    \mu = \tfrac{1}{\sigma^2} \Sigma X^\top y.
\end{align}
Or, equivalently, in the negative log-posterior form,
\begin{align}
    U(\theta) = \tfrac{1}{2\sigma^2}\|y - X\theta\|^2 + \tfrac{\tau}{2}\|\theta\|^2.
\end{align}

\begin{figure}[!b]
    \centering
    \includegraphics[width=1.\linewidth]{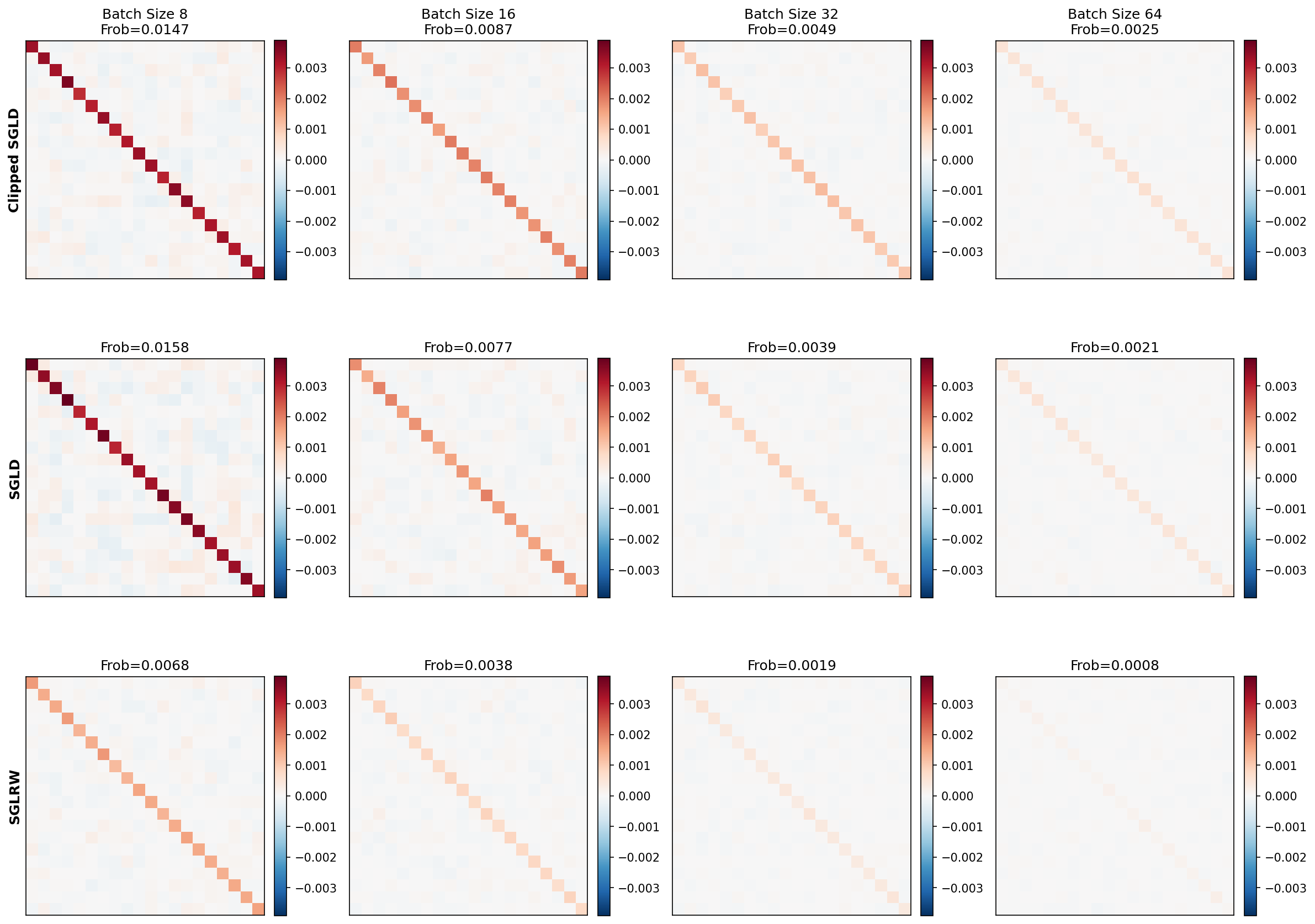}
    \caption{Covariance difference matrices $\Sigma_{\text{est}}-\Sigma_{\text{true}}$ for Bayesian linear regression at stepsize $\delta_t=10^{-3}$, shown across increasing minibatch sizes $B$. \textbf{Top:} Clipped SGLD. \textbf{Middle:} SGLD. \textbf{Bottom:} SGLRW. Each panel visualizes the deviation of the empirical posterior covariance from the analytic posterior covariance; the Frobenius norm (Frob) reports the total error magnitude. We observe that the error in the diagonal terms of the estimated covariance matrices is lower for SGLRW than SGLD and Clipped-SGLD}
    \label{fig:covariance_err}
\end{figure}

\paragraph{Setup}
Synthetic data are generated with $N=1000$ and $d=20$, using $\theta^\ast \sim \mathcal{N}(0,I)$, $\sigma^2 = 1.5$, and $\tau = 10^{-2}$.
Each method is run with $2{,}000$ parallel particles for $10{,}000$ iterations, using matched minibatch sizes and the same decaying learning-rate schedule.

As stated, the performance is quantified using the analytic Kullback--Leibler divergence between the true posterior $\mathcal{N}(\mu,\Sigma)$ and the empirical Gaussian fit to the samples, computed from their estimated mean and covariance.

\subsubsection{Results}
The KL curves reveal two consistent effects, portrayed in Table~\ref{tab:kl_comparison_bold}: (i) \emph{step size sensitivity:} SGLRW remains stable and continues to decrease KL under larger learning rates $\delta_t$ where SGLD diverges.
(ii) \emph{batch efficiency:} for comparable KL at matched $\delta_t$, SGLRW achieves the same accuracy with approximately half the minibatch size, indicating greater robustness to stochastic-gradient noise.

These trends are accompanied by different empirical covariance behaviour across methods, as illustrated in Figure~\ref{fig:covariance_err}. Here we can observe that the error in the diagonal terms of the estimated covariance matrices is significantly lower for SGLRW than SGLD and Clipped-SGLD, and in general less impacted by the batch size.

\subsection{UCI Bayesian Logistic Regression}\label{subsec:EXP_Bayesian_Logistic_Regression}
We now compare the sensitivity of SGLRW, SGLD and Clipped-SGLD on a non-Gaussian posterior sampling task, specifically logistic regression with the breast cancer dataset~\citep{breast_cancer_wisconsin}. The UCI breast cancer dataset consists of 569 samples, 30 features and 2 classes, which results in a 31-dimensional posterior distribution. 

\paragraph{Setup}
Since the true posterior is not available analytically, we compare to a gold-standard sample generated with NUTS~\citep{hoffman2014no} via Pyro~\citep{bingham2019pyro}. Throughout, we use a standard Gaussian prior on all parameters. In all cases, we ran 5{,}000 parallel chains for 1{,}000 steps, retaining only the final sample of each chain. All runs are averaged over 5 seeds.

\subsubsection{Results}
Comparing the inferred KL divergence of the three different methods in Table~\ref{tab:kl_comparison_logistic}, we see that SGLRW consistently outperforms SGLD and Clipped-SGLD across learning-rate settings, similar to what was observed in the linear regression experiment. However, in contrast to the linear regression experiments where Clipped-SGLD performed roughly similarly to SGLD and significantly worse than SGLRW, for the logistic regression experiment considered here Clipped-SGLD shows itself as a strong baseline. In the small-batch limit, where we observe a complete failure of SGLD, Clipped-SGLD performs only slightly worse than SGLRW, and occasionally better as the batch size increases.

\begin{figure*}[t]
    \centering
    \includegraphics[
      width=0.98\linewidth,
      trim=0pt 0pt 0pt 22pt,
      clip
    ]{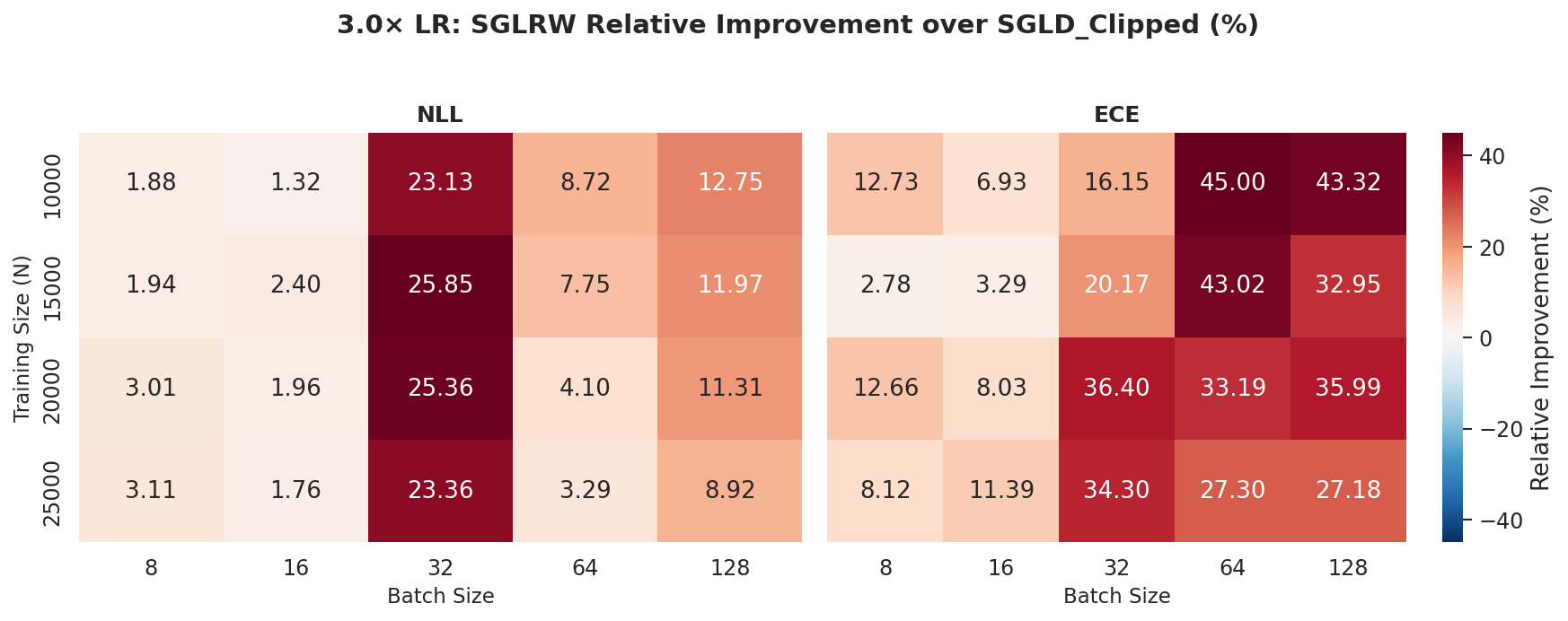}
    \caption{Relative improvement of SGLRW over clipped SGLD at increased learning-rate scale ($\eta_0 = 1.5 \times 10^{-4}$). Heatmaps show percentage differences in negative log-likelihood (\textbf{left}) and expected calibration error (\textbf{right}) across training-set sizes and minibatch sizes.}
    \label{fig:relative_improvement_3x}
\end{figure*}

\subsection{Sentiment Classification With LLM Features}
\label{sec:LMM}
Having considered two tasks with well-understood posterior distributions, we now turn to a more realistic problem where in practice issues arise due to the model size impacting the possible size of the minibatch: language modelling using LLMs.
Specifically, we evaluate SGLRW on a sentiment classification task using the IMDB dataset \citep{maas2011learning}, following a setup similar to \citet{harrison2024variational}. The dataset consists of 50,000 strongly polarized movie reviews, split evenly into training and test sets. To study the effect of data scale, we additionally consider subsampled training sets of varying sizes.

\begin{table}[t]
\vspace{0.3cm}
\centering
\caption{Inferred Kullback–Leibler divergence for the logistic regression problem. The KL divergence is measured between Gaussian distributions fitted to the empirical mean and covariance of a gold-standard reference sample and those obtained by each algorithm under the specified hyperparameters. 
\textbf{Bold} indicates the lowest KL divergence for a given $(B,\delta_t)$ pair.}
\label{tab:kl_comparison_logistic} 
\begin{small}
\setlength{\tabcolsep}{4pt}
\begin{tabular}{r c c c c}
\hline
\multicolumn{2}{c}{\textbf{Hyperparameters}} & \multicolumn{3}{c}{\textbf{KL Divergence}} \\
\cline{1-2} \cline{3-5}
$B$ & $\delta_t$ & \textbf{SGLD} & \textbf{SGLRW} & \textbf{Clipped SGLD} \\
\hline
1 & $10^{0}$ & inf & \textbf{8.3504} & 9.3560 \\
1 & $10^{-1}$ & 16.6812 & \textbf{6.0144} & 6.4732 \\
1 & $10^{-2}$ & 10.3472 & \textbf{3.9594} & 4.2549 \\
\hline
2 & $10^{0}$ & inf & \textbf{7.6706} & 9.1187  \\
2 & $10^{-1}$ & 8.9809 & \textbf{5.0197} & 5.5856  \\
2 & $10^{-2}$ & 5.4982 & \textbf{2.6152} & 2.8710  \\
\hline
4 & $10^{0}$ & 27.3098 & \textbf{6.9768} & 8.8915 \\
4 & $10^{-1}$ & 3.7046 & \textbf{3.6632} & 4.2974 \\
4 & $10^{-2}$ & 3.0155 & \textbf{1.4698} & 1.6333 \\
\hline
8 & $10^{0}$ & 19.0951 & \textbf{6.0397} & 8.4212 \\
8 & $10^{-1}$ & \textbf{1.3149} & 1.9429 & 2.4981 \\
8 & $10^{-2}$ & 1.9578 & \textbf{1.0051} & 1.0553 \\
\hline
16 & $10^{0}$ & 12.0667 & \textbf{4.5631} & 7.0883  \\
16 & $10^{-1}$ & \textbf{0.4993} & 0.9423 & 1.2400  \\
16 & $10^{-2}$ & 1.5158 & 0.8059 & \textbf{0.8027}  \\
\hline
32 & $10^{0}$ & 6.7814 & \textbf{2.8940} & 4.9928 \\
32 & $10^{-1}$ & \textbf{0.2506} & 0.4538 & 0.5611 \\
32 & $10^{-2}$ & 1.3235 & 0.6845 & \textbf{0.6629} \\
\hline
64 & $10^{0}$ & 3.4559 & \textbf{1.7535} & 3.0879 \\
64 & $10^{-1}$ & \textbf{0.1769} & 0.2153 & 0.2417 \\
64 & $10^{-2}$ & 1.2267 & 0.6490 & \textbf{0.6169} \\
\hline
\end{tabular}
\end{small}
\end{table}

\paragraph{Setup}
For each experiment, we extract fixed sequence embeddings from a pretrained OPT language model \citep{zhang2022opt} with 350M parameters by taking the final-layer representation of the last token. These embeddings are held fixed, and Bayesian posterior sampling is performed over the parameters of a two-layer binary classification head. This isolates the behaviour of the sampling algorithms from learning the data representation.

Each method is run with 15 parallel chains for 10{,}000 iterations, discarding the first 5{,}000 iterations as burn-in. For each training-set size, we vary the minibatch size to probe batch-size sensitivity while keeping learning-rate schedules and other hyperparameters matched across methods.

Performance is evaluated on the held-out test set using classification accuracy, negative log-likelihood (NLL), and expected calibration error (ECE).

\subsubsection{Results}
As highlighted previously, SGLRW has consistently been less sensitive to the choice of learning-rate schedule than standard SGLD; it can handle substantially larger step sizes while still maintaining stability across all training-set sizes and minibatch configurations. As such, we first compare the performance of SGLRW against standard SGLD and Clipped-SGLD with small initial step sizes. Following this, we explore the other end of the spectrum, where we compare the performance of SGLRW against Clipped-SGLD at step sizes for which standard SGLD is unstable.

\paragraph{Comparison at small step sizes.}
Fig.~\ref{fig:batchsize_25k} shows predictive accuracy and negative log-likelihood with respect to minibatch size for a large training-set size ($N=25{,}000$) for a run with $\eta_0 = 7.5 \times 10^{-6}$. At small minibatch sizes, SGLD exhibits a degradation in both accuracy and NLL, whereas SGLRW remains stable across the sweep. As the minibatch size increases, the accuracy of SGLD improves, while differences in NLL persist at moderate batch sizes. Similar to SGLRW, Clipped-SGLD also outperforms SGLD in this regime, again highlighting the strength of the baseline.

\begin{figure}
    \centering
    \includegraphics[width=\linewidth]{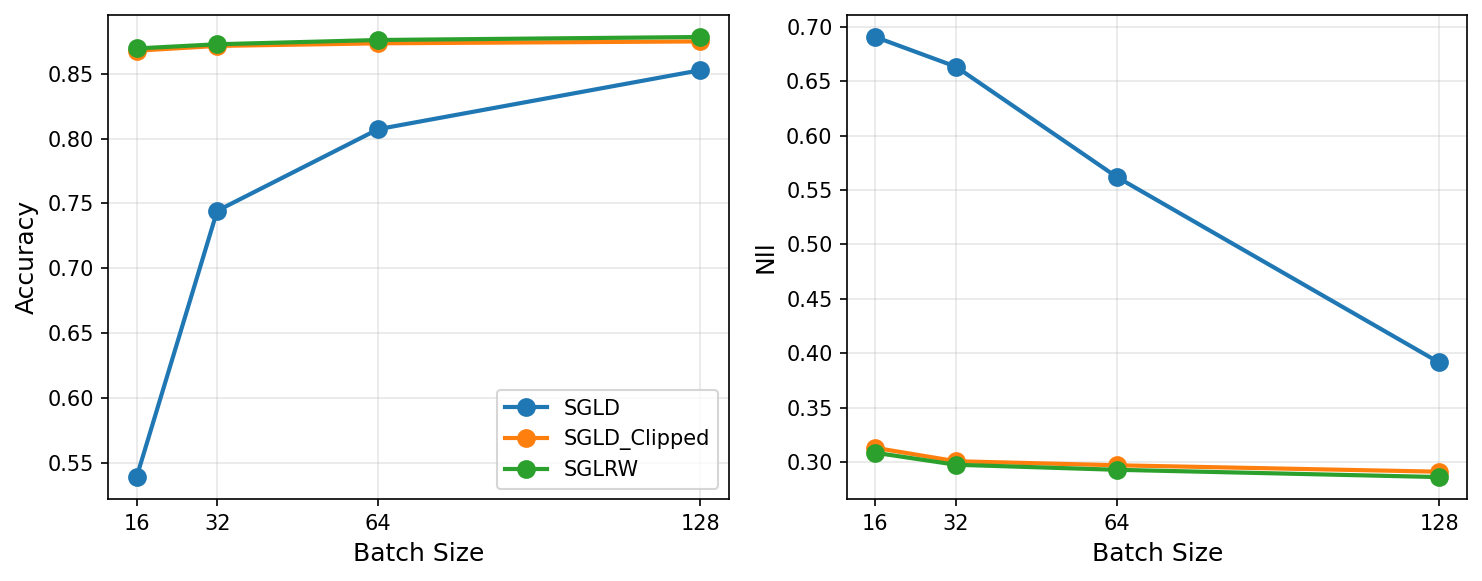}
    \caption{Predictive accuracy and negative log-likelihood (NLL) as a function of minibatch size for a large training-set size ($N=25{,}000$), using base learning schedule with $\eta_0 = 7.5 \times 10^{-6}$.}
    \label{fig:batchsize_25k}
\end{figure}

\paragraph{Clipped-SGLD versus SGLRW at larger step sizes.}
We next compare Clipped-SGLD and SGLRW in regimes where standard SGLD is unstable, focusing on step sizes beyond the conservative regime. Across all batch sizes considered in Figure~\ref{fig:relative_improvement_3x}, SGLRW consistently outperforms Clipped-SGLD in terms of predictive quality, with this behaviour remaining robust across training-set sizes.

The relative advantage of SGLRW becomes most pronounced in small-to-moderate minibatch regimes. While accuracy differences remain minor, SGLRW consistently attains lower negative log-likelihood and improved calibration relative to Clipped-SGLD. A representative comparison at an increased learning-rate scale is shown in Figure~\ref{fig:relative_improvement_3x}, with complete results across learning-rate schedules reported in Appendix~\ref{app:sentiment-classification} (Figures~\ref{fig:acc_heatmaps_350m}, \ref{fig:nll_heatmaps_350m}, and~\ref{fig:ece_heatmaps_350m}).
\vspace{-0.2cm}
\section{Conclusion}
This work introduced Stochastic Gradient Lattice Random Walk (SGLRW), a robust discretisation of Langevin dynamics for Bayesian inference. By replacing traditional Gaussian increments with coordinate-wise bounded updates, SGLRW significantly reduces sensitivity to minibatch size and stochastic gradient noise, a common failure point for standard SGLD. Our theoretical analysis demonstrated that SGLRW achieves strictly tighter mean squared error (MSE) bounds than SGLD by confining minibatch-induced noise to the off-diagonal elements of the update covariance.

Empirically, SGLRW showed superior stability and predictive performance across diverse tasks, from linear regression to LLM-based sentiment classification. This is in comparison to both standard SGLD as well as a strong baseline in the form of Clipped-SGLD. Notably, it remains stable under various conditions where SGLD diverges and maintains high calibration even with small minibatches.

Beyond its algorithmic advantages, the structure of SGLRW makes it uniquely suited for implementation on energy-efficient, low-precision, and stochastic hardware~\citep{duffield2025lattice}, which is becoming increasingly important as the impact of AI on energy consumption and sustainability becomes a major concern \citep{aifer2025solving}.

\section*{Impact Statement}
This work introduces a posterior sampling method that is naturally compatible with stochastic hardware and as such SGLRW may reduce energy and memory requirements in practical implementations. Beyond this, this paper presents work whose goal is to advance the field of machine learning and as such has many potential societal consequences.

\bibliography{shared}
\bibliographystyle{icml2026}

\appendix
\onecolumn
\section{Appendix}

\subsection{Analysis plots including Clipped SGLD}
\label{app:clippedSGLDAppendix}
For completeness, we also include Clipped SGLD in the plots for the analysis.

For the covariance MSE analysis (Figure~\ref{fig:mseplots}), Clipped SGLD already mitigates the sharp error explosion observed for SGLD at larger stepsizes (Figure~\ref{fig:fullMSEs}). However, across batch sizes and learning rates, SGLRW consistently attains comparable or lower covariance MSE throughout the stable regime.
\begin{figure}[!b]
    \centering
    \includegraphics[width=\linewidth]{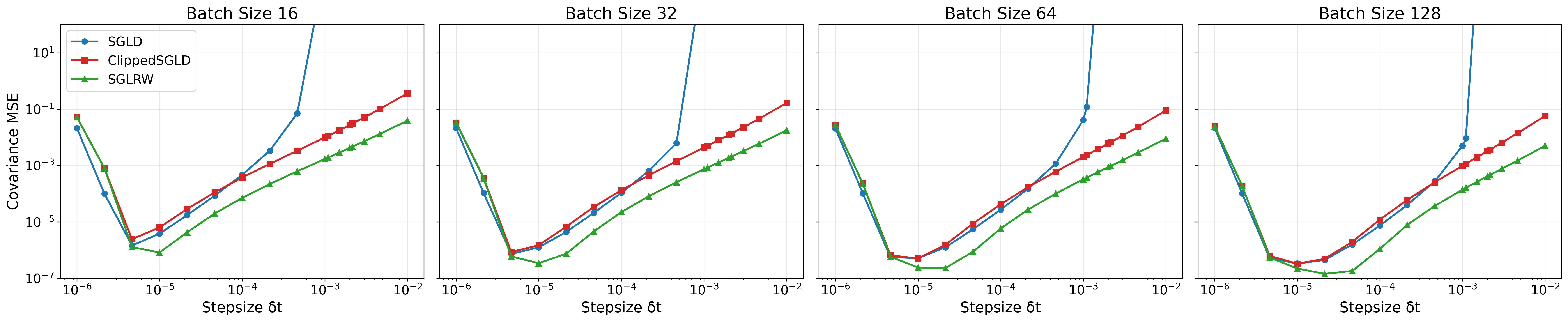}
    \includegraphics[width=\linewidth]{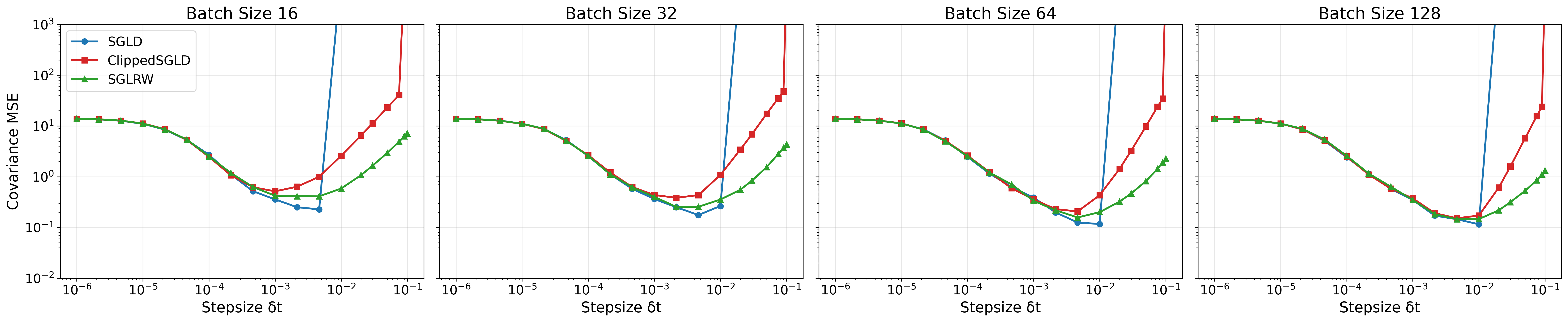}
    \caption{Mean-squared error (MSE) of the posterior covariance as a function of the stepsize $\delta_t$, shown for different batch sizes. \textbf{Top:} $50$-dimensional Bayesian linear regression. \textbf{Bottom:} Bayesian logistic regression on the breast cancer dataset \citep{breast_cancer_wisconsin}.}
    \label{fig:fullMSEs}
\end{figure}

For the heavy-tailed noise robustness analysis (Figure~\ref{fig:1d}), both Clipped SGLD and SGLRW prevent the severe instability exhibited by SGLD under heavy-tailed gradient noise. Across noise scales, SGLRW maintains a closer qualitative agreement with the target distribution than Clipped SGLD.
\begin{figure}[!b]
    \centering
    \includegraphics[width=0.47\linewidth]{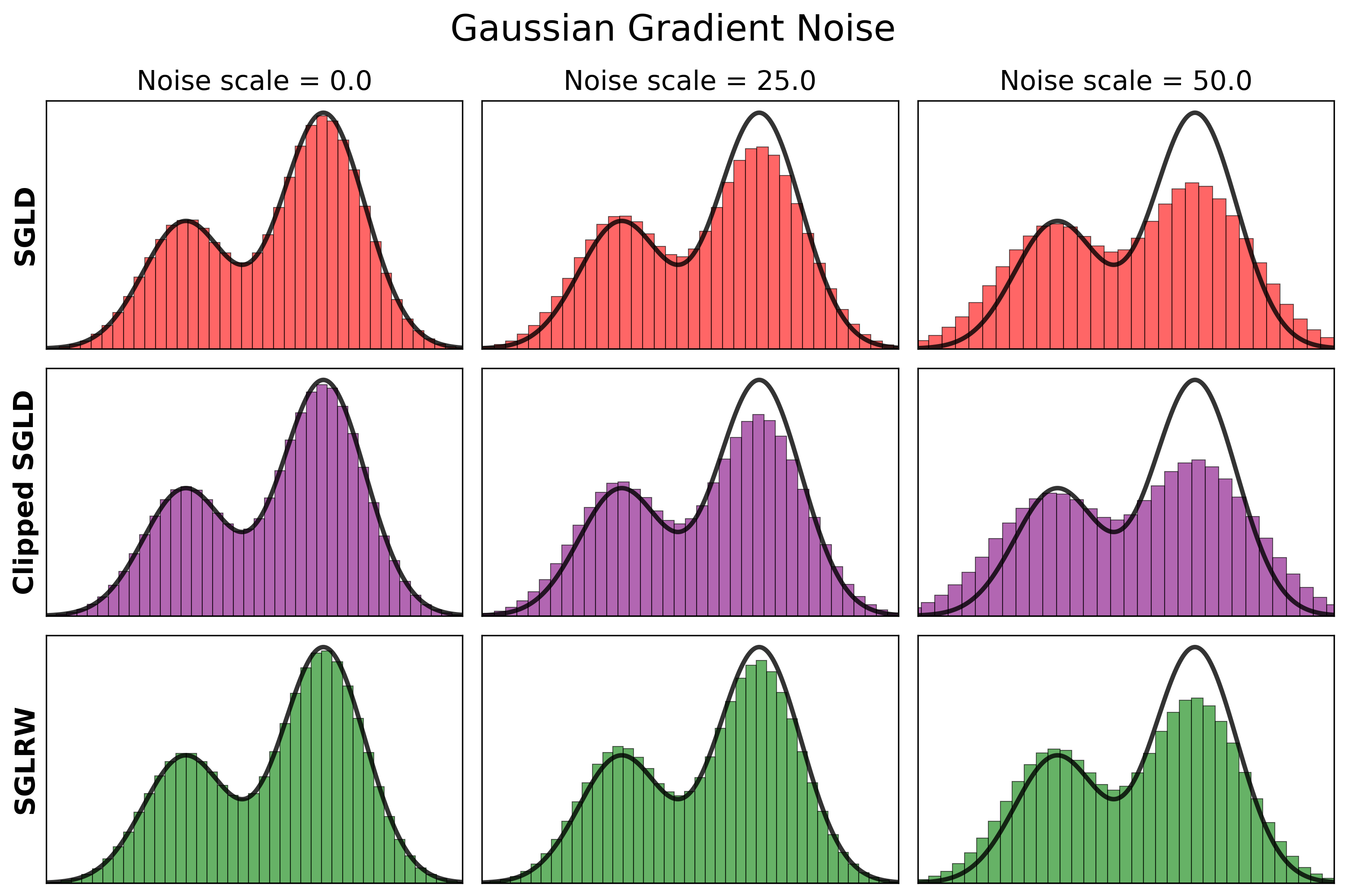}
    \qquad
    \includegraphics[width=0.47\linewidth]{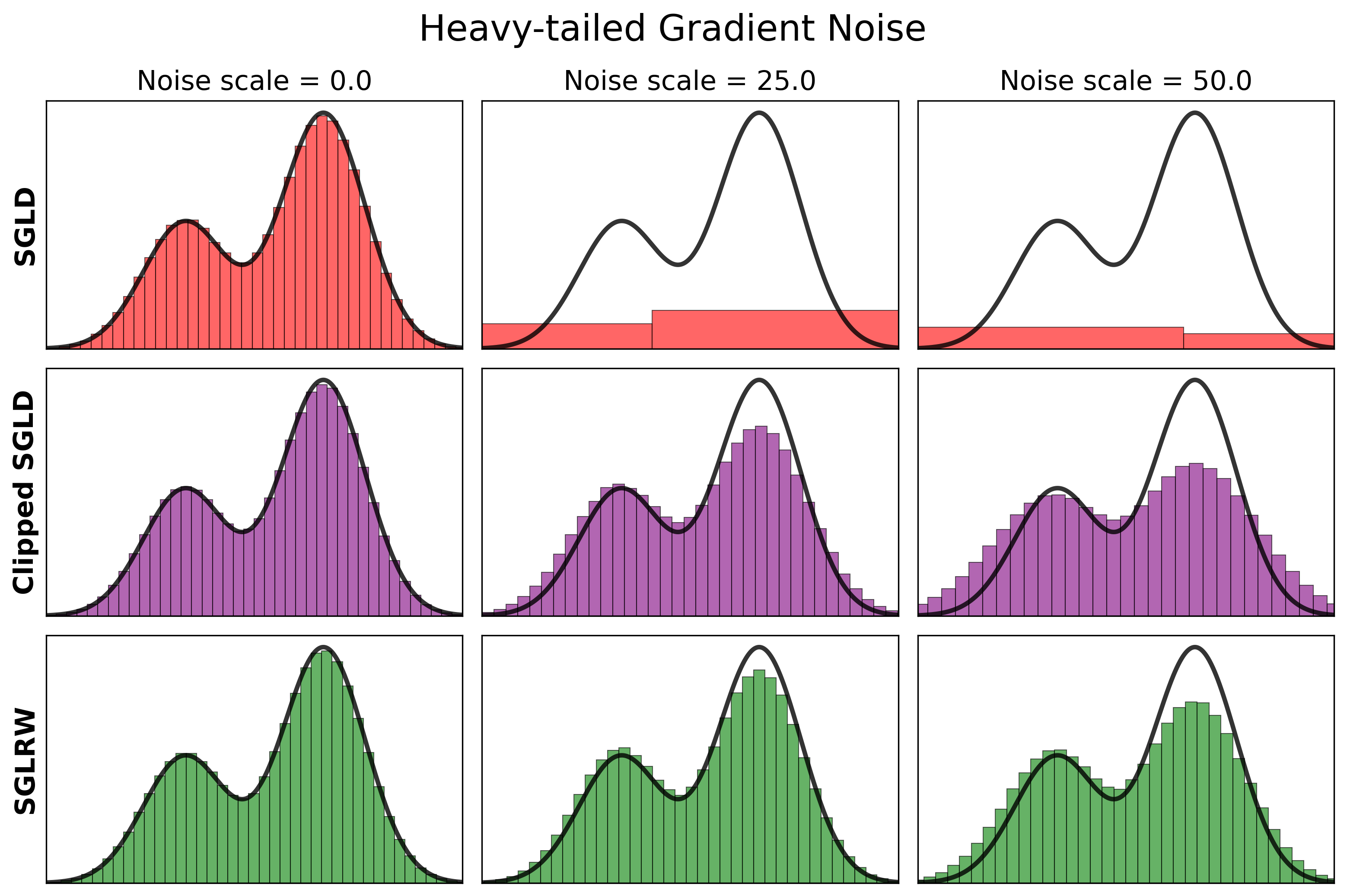}
    \caption{Multimodal univariate target with exact gradient corrupted by synthetic noise of increasing scale. \textbf{Left:} Gaussian noise ($\alpha=2$). \textbf{Right:} heavy-tailed noise drawn from an $\alpha$-stable distribution with $\alpha=1.5$.}
    \label{fig:fullheavies}
\end{figure}

\subsection{Proof of Theorem~\ref{thm:refined-mse-bound}} \label{app:k1-proof}
In this appendix we prove Theorem~\ref{thm:refined-mse-bound}. The argument follows the Poisson-equation framework of \cite{chen2015convergence}; we restate the required notation so the appendix is self-contained.

\paragraph{Generator and Kolmogorov operators.}
Consider a continuous-time It\^o diffusion on $\mathbb R^d$ with infinitesimal generator
\begin{align}
\mathcal L g(\theta) = f(\theta)\cdot\nabla g(\theta) + D(\theta) : \nabla^2 g(\theta),
\end{align}

with $f$ the drift of the SDE in \eqref{eq:overdamped_general_sde}.
Let $(e^{t\mathcal L})_{t\ge 0}$ denote the associated Kolmogorov (backward) semigroup, so that for any suitable test function $g$,
\begin{align}
\mathbb E[g(\theta_t)\mid \theta_0=\theta] = (e^{t\mathcal L}g)(\theta).
\end{align}

Since $e^{t\mathcal L}$ is generally intractable, we consider a time-$\delta_t$ numerical update with one-step Markov operator $P_{\delta_t}^{(\mathcal L)}$ defined by
\begin{align}
\mathbb E[g(\theta_{n\delta_t})\mid \theta_{(n-1)\delta_t}] = (P_{\delta_t}^{(\mathcal L)}g)(\theta_{(n-1)\delta_t}).
\end{align}
A one-step scheme $P_{\delta_t}$ is a weak order-$K$ local integrator if, for all sufficiently smooth $g$, we have
\begin{align}
(P_{\delta_t}g)(\theta) = (e^{\delta_t\mathcal L}g)(\theta) + O(\delta_t^{K+1}).
\end{align}

\paragraph{Stochastic gradients and random one-step operators.}
The Euler-Maruyama discretization scheme is known to be weak order-one. We now examine the effect of using minibatches. In SGLD as well as SGLRW, the exact drift $f$ is replaced by a minibatch approximation $\widehat{f}$. Let $\widetilde P_{\delta_t}^{(n)}$ denote the resulting (random) one-step Markov operator at iteration $n$, i.e.
\begin{align}
\mathbb E[g(\theta_{n\delta_t})\mid \theta_{(n-1)\delta_t}] = (\widetilde P_{\delta_t}^{(n)} g)(\theta_{(n-1)\delta_t}).
\end{align}
For a given minibatch $\mathcal B$, define
\begin{align}
\zeta(\theta; \mathcal B) := f(\theta) - \widehat{f}(\theta;\mathcal{B}).
\end{align}
and the associated first-order differential operator
\begin{align}
(\Delta V_n g)(\theta ; \mathcal B) := \zeta(\theta;\mathcal B_n)\cdot\nabla g(\theta).
\end{align}
capturing the error due to the minibatch drift error. Note that for overdamped Langevin diffusion, we have $f = -\nabla U$, $\widehat{f}(\theta;\mathcal{B}) = -\widehat{\nabla U}(\theta;\mathcal B)$, so that $\zeta(\theta; \mathcal B) = \widehat{\nabla U}(\theta;\mathcal B) - \nabla U(\theta)$.

Beyond the first-order drift perturbation, minibatching can induce a second-order correction through the conditional second moment of the increment. We define the second-order differential operator
\begin{equation}\label{eq:app_defAn}
(\Delta A_n g)(\theta) := \frac12\,M_n(\theta):\nabla^2 g(\theta),
\end{equation}
with
\begin{equation}
M_n(\theta,\mathcal B_n) :=\delta_t^{-2}\, \mathbb{E}_{\xi_n}\!\big[  \Delta\theta_n^{\mathrm{fb}}(\Delta\theta_n^{\mathrm{fb}})^\top - \Delta\theta_n^{\mathrm{mb}}(\Delta\theta_n^{\mathrm{mb}})^\top \;\big|\; \theta_{(n-1)\delta_t}=\theta,\mathcal B_n \big].
\end{equation}
Here $\Delta\theta_n^{\mathrm{mb}}$ and $\Delta\theta_n^{\mathrm{fb}}$ denote the one-step increments of the minibatch and full-batch updates, respectively, and $\mathbb{E}_{\xi_n}[\cdot]$ denotes expectation with respect to the internal randomness of the update at step $n$ (e.g.\ injected Gaussian noise or lattice path sampling).

\paragraph{Poisson equation.}
Given a smooth observable $\phi:\mathbb R^d\to\mathbb R$, define its stationary expectation under the invariant distribution $\pi$ by
\begin{align}
\bar\phi := \int \phi(\theta)\,\pi(\theta)\, d\theta.
\end{align}
We analyze ergodic averages via the Poisson equation
\begin{align}
\mathcal L\psi = \phi-\bar\phi,
\end{align}
and express finite-time errors in terms of the corresponding solution $\psi$. We assume that, for the Poisson solution $\psi$, the remainder satisfies the Lyapunov-weighted bound, in the sense that there exists a constant $p_0>0$ such that
\begin{equation} \label{eq:Rpsi-bound}
|\mathcal R_n\psi(\theta)| \le C\,\mathcal V(\theta)^{p_0}, \qquad \text{uniformly in } n \text{ and } \delta_t\in(0,1].
\end{equation}

\paragraph{Lyapunov-Poisson regularity.}
\begin{assumption}
\label{ass:lyap}
Let $\psi$ solve the Poisson equation $\mathcal L\psi=\phi-\bar\phi$.
Assume there exists a function $\mathcal V:\mathbb R^d\to[1,\infty)$ such that:
\begin{enumerate}[(i)]
\item (\emph{Derivative control}) There exist constants $C_k,p_k>0$ for $k=0,1,2,3,4$ such that
\begin{equation}\label{eq:dercontrol}
\|D^k\psi(\theta)\|\le C_k\,\mathcal V(\theta)^{p_k}.
\end{equation}
\item (\emph{Uniform moments along the chain}) There exist constants $p^*$ such that for all $p\le p^\ast$,
\begin{equation}
\sup_{n}\mathbb E[\mathcal V(\theta_{n\delta_t})^p]<\infty.
\end{equation}
\item (\emph{Growth compatibility}) For the constant $p^\ast$ and some $C > 0$, for all $p\le p^\ast$ and all $s\in(0,1)$,
\begin{equation}
\mathcal V^p(s\theta+(1-s)\vartheta)\le C\bigl(\mathcal V^p(\theta)+\mathcal V^p(\vartheta)\bigr).
\end{equation}
\item (\emph{Uniform second-moment bound}) The second-order coefficient field satisfies
\begin{equation}
\sup_n \mathbb E\!\left[ \|M_n(\theta_{(n-1)\delta_t})\|_F^2 \right] <\infty.
\end{equation}
\item (\emph{Increment moment control}) There exist constants $C>0$ and exponents $q_2,q_4>0$ such that, for all $n$,
\begin{equation}\label{eq:inc-moment-control}
\mathbb{E}_{\xi_n, \mathcal B_n}\!\left[\|\Delta\theta_n\|^{2k} \,\middle|\, \theta_{(n-1)\delta_t}\right] \le C\,\delta_t^{k}\,\mathcal V(\theta_{(n-1)\delta_t})^{q_{2k}},
\qquad k\in\{1,2\},
\end{equation}
where $\Delta\theta_n:=\theta_{n\delta_t}-\theta_{(n-1)\delta_t}$.

\item (\emph{Third-moment tensor control}) 
There exist constants $C>0$ and an exponent $q_3>0$ such that, for all $n$,
\begin{equation}\label{eq:inc-third-tensor-diff}
\Big\| \mathbb{E}_{\xi_n, \mathcal B_n}\!\left[ (\Delta\theta_n^{\mathrm{mb}})^{\otimes 3} \;\middle|\; \theta_{(n-1)\delta_t} \right] - \mathbb{E}_{\xi_n}\!\left[ (\Delta\theta_n^{\mathrm{fb}})^{\otimes 3} \;\middle|\; \theta_{(n-1)\delta_t} \right] \Big\| \le C\,\delta_t^{2}\, \mathcal V(\theta_{(n-1)\delta_t})^{q_3}.
\end{equation}
\item (\emph{Unbiased minibatch estimator})
The minibatch drift estimator is conditionally unbiased, in the sense that
\begin{equation}
\mathbb E_{\mathcal B_n}\!\left[\zeta(\theta_{(n-1)\delta_t};\mathcal B_n)\,\middle|\,\theta_{(n-1)\delta_t}\right]=0
\qquad \text{a.s. for all } n,
\end{equation}
where $\zeta(\theta;\mathcal B):=f(\theta)-\widehat f(\theta;\mathcal B)$ denotes the minibatch drift error.
\end{enumerate}
\end{assumption}
Note that the derivative bounds \eqref{eq:dercontrol} in Assumption~\ref{ass:lyap}  can be verified by constructing a Lyapunov function $\mathcal V:\mathbb R^d\to[1,\infty)$ which tends to infinity as $\theta\to\infty$, is twice continuously differentiable with bounded second derivatives, and satisfies the following conditions, as shown in the Appendix~C of \cite{chen2015convergence}:
\begin{itemize}
\item[(a)] (\emph{Lyapunov drift condition}) There exist constants $\alpha,\beta>0$ such that the exact drift field $f$ satisfies
\begin{equation}
\langle \nabla \mathcal V(\theta), f(\theta) \rangle \le -\alpha\,\mathcal V(\theta)+\beta .
\end{equation}

\item[(b)] (\emph{Minibatch-induced drift fluctuations}) There exists $p_H\ge 2$ such that
\begin{equation}
\mathbb{E}_{\xi_n, \mathcal B_n}\!\left[ \|\zeta(\theta_{(n-1)\delta_t}; \mathcal B_n)\|^2 \;\middle|\; \theta_{(n-1)\delta_t}=\theta \right] \le C\,\mathcal V(\theta)^{p_H},
\end{equation}
together with the growth condition
\begin{equation}
\|\nabla\mathcal V(\theta)\|^2+\|f(\theta)\|^2 \le C\,\mathcal V(\theta).
\end{equation}
\end{itemize}
\begin{remark}
Throughout, we write $\theta_{(n-1)\delta_t}=\theta$ and recall that for overdamped Langevin diffusion we have $\zeta(\theta; \mathcal B) = \widehat{\nabla U}(\theta;\mathcal B) - \nabla U(\theta)$. Assumptions~\ref{ass:lyap}(v)–(vi) hold for SGLD and SGLRW for varying requirements on the minibatch noise $\zeta$.
\begin{enumerate}
\item \textbf{SGLD.}
\begin{itemize}
\item \emph{Second moment ($k=1$):} From $\Delta\theta_n^{\mathrm{mb}} =-\delta_t(\nabla U(\theta)+\zeta)+\sqrt{2\delta_t}\,\xi_n$, and using $\| a+b\|^2\leq 2\|a\|^2 + 2 \|b\|^2 $ we have
\begin{align}
\mathbb{E}_{\xi_n, \mathcal B_n} \bigl[\|\Delta\theta_n^{\mathrm{mb}}\|^2 \mid \theta\bigr] \le 2\delta_t^2\, \mathbb E_{\mathcal B_n} \bigl[\|\nabla U(\theta)+\zeta\|^2 \mid \theta \bigr] +4d\,\delta_t .
\end{align}
Thus, assuming
\begin{align}
\mathbb E_{\mathcal B_n} \bigl[\|\zeta\|^2 \mid \theta\bigr] \le C\,\mathcal V(\theta)^{q_2},
\qquad \|\nabla U(\theta)\|^2\le C\,\mathcal V(\theta),
\end{align}
yields
\begin{align}
\mathbb E_{\mathcal B_n,\varepsilon_n} \bigl[\|\Delta\theta_n^{\mathrm{mb}}\|^2 \mid \theta\bigr] =O(\delta_t).
\end{align}

\item \emph{Fourth moment control ($k=2$):} Similarly,
\begin{align}
\mathbb{E}_{\xi_n, \mathcal B_n} \bigl[\|\Delta\theta_n^{\mathrm{mb}}\|^4 \mid \theta\bigr] \le C\delta_t^4\, \mathbb E_{\mathcal B_n} \bigl[\|\nabla U(\theta)+\zeta\|^4 \mid \theta\bigr] +C\delta_t^2 .
\end{align}
Hence, if
\begin{align}
\mathbb E_{\mathcal B_n} \bigl[\|\zeta\|^4 \mid \theta\bigr] \le C\,\mathcal V(\theta)^{q_4},
\end{align}
then
\begin{align}
\mathbb{E}_{\xi_n, \mathcal B_n} \bigl[\|\Delta\theta_n^{\mathrm{mb}}\|^4 \mid \theta\bigr] =O(\delta_t^2).
\end{align}

\item \emph{Third-moment control:} From the expansion in \eqref{eq:app-sgld-third-moment}, the difference in \eqref{eq:inc-third-tensor-diff} arises from cubic drift noise interaction terms, and scales as $O(\delta_t^3)$. Assuming
\begin{align}
\mathbb E_{\mathcal B_n} \bigl[\|\zeta\|^3 \mid \theta\bigr] \le C\,\mathcal V(\theta)^{q_3},
\end{align}
this difference is bounded by $C\,\delta_t^2\,\mathcal V(\theta)^{q_3}$ for $\delta_t\in(0,1]$.
\end{itemize}

\item \textbf{SGLRW.}
\begin{itemize}
\item \emph{Increment control ($k=1,2$):} As $\|\Delta\theta_n^{\mathrm{mb}}\|^2=2d\,\delta_t$ and $\|\Delta\theta_n^{\mathrm{mb}}\|^4=4d^2\,\delta_t^2$, the bound \eqref{eq:inc-moment-control} holds without any assumption on $\zeta$.
\item \emph{Third-moment control:} As shown in \eqref{eq:app-sglrw-third-moment}, the third-order conditional moments are batch-independent when at least two of the three directions coincide.

For distinct indices $(i,j,k)$, the difference
is $O(\delta_t^3)$. Consequently, \eqref{eq:inc-third-tensor-diff} holds provided
\begin{align}
\mathbb E_{\mathcal B_n} \bigl[\|\zeta\|^3 \mid \theta\bigr] \le C\,\mathcal V(\theta)^{q_3}.
\end{align}
\end{itemize}
\end{enumerate}
\end{remark}

\iffalse
\begin{lemma}[Refined weak expansion] \label{lem:refined-weak-expansion}
Under Assumption~\ref{ass:lyap}, the random one-step Markov operator $\widetilde P_{\delta_t}^{(n)}$ associated with a weak order-one integrator admits the expansion
\begin{equation}
\widetilde P_{\delta_t}^{(n)}\psi(\theta) = \psi(\theta) + \delta_t(\mathcal L-
\Delta V_n)\psi(\theta) - \delta_t^2\,\Delta A_n\psi(\theta) + \delta_t^2\,\mathcal R_n\psi(\theta),
\end{equation}
where $\Delta V_n$ is the (minibatch-dependent) first-order drift perturbation, $\Delta A_n$ is the second-order covariance perturbation, and $|\mathcal R_n\psi(\theta)| \le C\,\mathcal V(\theta)^{p_0}$ uniformly in $n$ and $\delta_t\in(0,1]$.
\end{lemma}
\fi
In SGLD as well as SGLRW, the exact drift $f$ is replaced by a minibatch approximation $\widehat{f}$. Since the underlying Euler-Maruyama scheme is a weak order-one integrator, we can derive a refined weak expansion for the resulting random operator, explicitly capturing the drift and covariance perturbations. 
\begin{lemma}[Refined weak expansion] \label{lem:refined-weak-expansion}
Consider a numerical integrator that is weak order-one for the full batch dynamics. Let $\widetilde P_{\delta_t}^{(n)}$ denote the random operator obtained by replacing the exact drift with a mini-batch estimator, for which Assumption A.1 holds. Then $\widetilde P_{\delta_t}^{(n)}$ admits the expansion:
\begin{equation}
\widetilde P_{\delta_t}^{(n)}\psi(\theta) = \psi(\theta) + \delta_t(\mathcal L-
\Delta V_n)\psi(\theta) - \delta_t^2\,\Delta A_n\psi(\theta) + \delta_t^2\,\mathcal R_n\psi(\theta),
\end{equation}
where $\Delta V_n$ is the first-order drift perturbation, $\Delta A_n$ is the second-order covariance perturbation, and $|\mathcal R_n\psi(\theta)| \le C\,\mathcal V(\theta)^{p_0}$ uniformly in $n$ and $\delta_t\in(0,1]$.
\end{lemma}

\begin{proof}
 By Taylor’s theorem with integral remainder applied to the minibatch increment,
\begin{align}
\psi(\theta+\Delta\theta_n^{\mathrm{mb}}) &= \psi(\theta) + \nabla\psi(\theta)\cdot\Delta\theta_n^{\mathrm{mb}} + \frac12\,\Delta\theta_n^{\mathrm{mb}} (\Delta\theta_n^{\mathrm{mb}})^\top :\nabla^2\psi(\theta) \nonumber\\
&\quad + \frac16\,\nabla^3\psi(\theta): (\Delta\theta_n^{\mathrm{mb}})^{\otimes 3} + R_4(\theta,\Delta\theta_n^{\mathrm{mb}}),
\end{align}
where
\begin{align}
R_4(\theta,\Delta\theta) = \frac16\int_0^1(1-s)^3 \nabla^4\psi(\theta+s\Delta\theta):\Delta\theta^{\otimes 4}\,ds.
\end{align}

For a given minibatch $\mathcal B_n$, taking the conditional expectation yields
\begin{align}
\widetilde P_{\delta_t}^{(n)}\psi(\theta) &= \psi(\theta) + \nabla\psi(\theta)\cdot \mathbb{E}_{\xi_n}[\Delta\theta_n^{\mathrm{mb}}\mid\theta] \nonumber + \frac12\, \mathbb{E}_{\xi_n}[\Delta\theta_n^{\mathrm{mb}} (\Delta\theta_n^{\mathrm{mb}})^\top\mid\theta] :\nabla^2\psi(\theta) \nonumber\\
&\quad + \frac16\,\nabla^3\psi(\theta): \mathbb{E}_{\xi_n}[(\Delta\theta_n^{\mathrm{mb}})^{\otimes 3}\mid\theta] + \mathbb{E}_{\xi_n}[R_4(\theta,\Delta\theta_n^{\mathrm{mb}})\mid\theta].
\end{align}

By definition of the increment-based drift difference,
\begin{align}
\zeta(\theta; \mathcal B_n) = f(\theta) - \widehat{f}(\theta;\mathcal{B}),
\end{align}
so
\begin{equation}
\nabla\psi(\theta)\cdot \mathbb{E}_{\xi_n}[\Delta\theta_n^{\mathrm{mb}}\mid\theta] = \delta _t\,\nabla \psi(\theta) \cdot \widehat{f}(\theta;\mathcal{B}) = \delta_t\,f(\theta)\cdot\nabla\psi(\theta) - \delta_t\,\Delta V_n\psi(\theta).
\end{equation}
Similarly, by definition of $\Delta A_n$,
\begin{align}
\frac12\, \mathbb{E}_{\xi_n}[\Delta\theta_n^{\mathrm{mb}} (\Delta\theta_n^{\mathrm{mb}})^\top\mid\theta] :\nabla^2\psi(\theta) &= \frac12\, \mathbb{E}_{\xi_n}[\Delta\theta_n^{\mathrm{fb}} (\Delta\theta_n^{\mathrm{fb}})^\top\mid\theta] :\nabla^2\psi(\theta)  - \delta_t^2\,\Delta A_n\psi(\theta).
\end{align}
We now treat the cubic term by adding and subtracting the full-batch third moment:
\begin{align}
\frac16\,\nabla^3\psi(\theta):
\mathbb{E}_{\xi_n}[(\Delta\theta_n^{\mathrm{mb}})^{\otimes 3}\mid\theta] &= \frac16\,\nabla^3\psi(\theta): \mathbb{E}_{\xi_n}[(\Delta\theta_n^{\mathrm{fb}})^{\otimes 3}\mid\theta] \nonumber\\
&\quad + \frac16\,\nabla^3\psi(\theta): \Big( \mathbb{E}_{\xi_n}[(\Delta\theta_n^{\mathrm{mb}})^{\otimes 3}\mid\theta] - \mathbb{E}_{\xi_n}[(\Delta\theta_n^{\mathrm{fb}})^{\otimes 3}\mid\theta] \Big).
\end{align}

Since the full-batch scheme is weak order one, its contribution is absorbed into the $O(\delta_t^2)$ remainder of the full-batch expansion. By Assumption~\ref{ass:lyap}(vi), the difference of third-order moments is $O(\delta_t^2\,\mathcal V(\theta)^{q_3})$, so the net cubic contribution is $O(\delta_t^2\,\mathcal V(\theta)^{p_0})$.

For the fourth-order remainder, using $|T:x^{\otimes 4}|\le \|T\|\,\|x\|^4$ and Assumption~\ref{ass:lyap}(i) with $k=4$, we obtain
\begin{align}
|R_4(\theta,\Delta\theta_n^{\mathrm{mb}})| \le C\,\|\Delta\theta_n^{\mathrm{mb}}\|^4 \int_0^1 \mathcal V(\theta+s\Delta\theta_n^{\mathrm{mb}})^{p_4}\,ds.
\end{align}
By Assumption~\ref{ass:lyap}(iii), $\mathcal V(\theta+s\Delta\theta_n^{\mathrm{mb}})^{p_4}\le C(\mathcal V(\theta)^{p_4}+\mathcal V(\theta+\Delta\theta_n^{\mathrm{mb}})^{p_4})$ for $s\in(0,1)$. Taking conditional expectation, applying the increment fourth-moment bound in Assumption~\ref{ass:lyap}(v), and using the uniform moment bound in Assumption~\ref{ass:lyap}(ii), yields
\begin{align}
\mathbb{E}_{\xi_n}[|R_4(\theta,\Delta\theta_n^{\mathrm{mb}})|\mid \theta, \mathcal B_n] \le C\,\delta_t^2\,\mathcal V(\theta)^{p_0},
\end{align}
uniformly in $n$ and $\delta_t\in(0,1]$.

Collecting all $O(\delta_t^2)$ contributions into $\mathcal R_n$ yields the stated expansion.
\end{proof}

\begin{theorem}\label{thm:app_refined-mse}
Under the Assumption~\ref{ass:lyap}, there exists $C>0$, independent of $(L,\delta_t)$, such that
\begin{equation}\label{eq:refined-mse}
\begin{aligned}
\mathbb{E}(\hat\phi-\bar\phi)^2 \le C\Bigg( &\frac{1}{L^2}\sum_{n=1}^L \mathbb{E}\big[\|\zeta_n\|^2\big] + \frac{1}{L\delta_t} + \frac{1}{L^2\delta_t^2} + \delta_t^2  + \frac{\delta_t^2}{L}\sum_{n=1}^L \mathbb{E}\big[\|M_n\|^2\big] \Bigg).
\end{aligned}
\end{equation}
\end{theorem}

\begin{proof}
First write 
$\hat\phi-\bar\phi ={1\over L}\sum_{n=1}^{L} {\mathcal L}\psi(\theta_{n\delta_t})$.  
Next, use Lemma~\ref{lem:refined-weak-expansion} to rewrite ${\mathcal L}\psi(\theta)$, and using $\mathbb{E}_{\xi_n}[\psi(\theta_{n\delta_t})\mid \theta_{(n-1)\delta_t}, \mathcal B_n] =(\widetilde P_{\delta_t}^{(n)}\psi)(\theta_{(n-1)\delta_t})$, we obtain
\begin{align}\label{eq:decomposition-refined}
\hat\phi-\bar\phi &= \frac{1}{L\delta_t}\Big(\mathbb \psi(\theta_{L\delta_t})-\psi(\theta_{\delta_t})\Big)-\frac{1}{L\delta_t}\sum_{n=1}^{L}\Big(\mathbb{E}_{\xi_n}[\psi(\theta_{n\delta_t}) \mid \theta_{(n-1)\delta_t}]-\psi(\theta_{n\delta_t})\Big) \nonumber\\
&\quad +\frac{1}{L}\sum_{n=1}^{L}\Delta V_n\psi(\theta_{(n-1)\delta_t}) +\frac{\delta_t}{L}\sum_{n=1}^{L}\Delta A_n\psi(\theta_{(n-1)\delta_t}) -\frac{\delta_t}{L}\sum_{n=1}^{L}\mathcal R_n\psi(\theta_{(n-1)\delta_t}).
\end{align}
Taking squares and using $(a+b+c+d+e)^2\le 5(a^2+b^2+c^2+d^2+e^2)$ gives
\begin{align}
\mathbb E(\hat\phi-\bar\phi)^2 \le C\,(\alpha_1+\alpha_2+\alpha_3+\alpha_4+\alpha_5),
\end{align}
where
\begin{align}
\alpha_1&:=\mathbb E\!\left[\left(\frac{\psi(\theta_{L\delta_t})-\psi(\theta_{\delta_t})}{L\delta_t}\right)^2\right],
\\
\alpha_2&:=\mathbb E\!\left[\frac{1}{L^2\delta_t^2}\left(\sum_{n=1}^L\big(\mathbb{E}_{\xi_n}[\psi(\theta_{n\delta_t})\mid \theta_{(n-1)\delta_t}] - \psi(\theta_{n\delta_t})\big)\right)^2\right],
\\
\alpha_3&:=\mathbb E\!\left[\frac{1}{L^2}\left(\sum_{n=1}^L \Delta V_n\psi(\theta_{(n-1)\delta_t})\right)^2\right],
\\
\alpha_4&:=\mathbb E\!\left[\left(\frac{\delta_t}{L}\sum_{n=1}^L \Delta A_n\psi(\theta_{(n-1)\delta_t})\right)^2\right],
\\
\alpha_5&:=\mathbb E\!\left[\left(\frac{\delta_t}{L}\sum_{n=1}^L \mathcal R_n\psi(\theta_{(n-1)\delta_t})\right)^2\right].
\end{align}
\begin{enumerate}
\item[\textbf{(i)}] \textbf{$\alpha_1$.}
By Lyapunov control of $\psi$ (Assumption~\ref{ass:lyap}(i--ii) with $k=0$), $\sup_n \mathbb E[\psi(\theta_{n\delta_t})^2]<\infty$, hence
\begin{align}
\alpha_1 \le \frac{C}{L^2\delta_t^2}.
\end{align}
\item[\textbf{(ii)}] \textbf{$\alpha_2$.}
Let
\begin{align}
Z_n := \mathbb{E}_{\xi_n}[\psi(\theta_{n\delta_t})\mid \theta_{(n-1)\delta_t}] - \psi(\theta_{n\delta_t}),
\end{align}
so that $\{Z_n\}_{n\ge1}$ is a martingale difference sequence and hence
\begin{align}
\alpha_2 = \frac{1}{L^2\delta_t^2} \sum_{n=1}^L \mathbb{E}\!\left[ \mathrm{Var}(\psi(\theta_{n\delta_t})\mid \theta_{(n-1)\delta_t}) \right].
\end{align}

Write $\theta:=\theta_{(n-1)\delta_t}$ and $\Delta\theta_n:=\theta_{n\delta_t}-\theta$. By the fundamental theorem of calculus,
\begin{equation}
\psi(\theta+\Delta\theta_n)-\psi(\theta) =\Delta\theta_n\cdot \int_0^1 \nabla\psi(\theta+s\Delta\theta_n)\,ds.
\end{equation}
Therefore,
\begin{align}
\mathrm{Var}(\psi(\theta_{n\delta_t})\mid\theta)
=\mathrm{Var}(\psi(\theta_{n\delta_t})-\psi(\theta)\mid\theta)
&\le \mathbb{E}\big[(\psi(\theta+\Delta\theta_n)-\psi(\theta))^2\mid\theta \big] \nonumber\\
&\le \mathbb{E}\left[ \|\Delta\theta_n\|^2 \int_0^1 \|\nabla\psi(\theta+s\Delta\theta_n)\|^2\,ds \;\middle|\;\theta \right],
\end{align}
where we used Cauchy-Schwarz.

By Assumption~\ref{ass:lyap}(i), $\|\nabla\psi(x)\|^2\le C\,\mathcal V(x)^{2p_1}$, and by Assumption~\ref{ass:lyap}(iii),
\begin{align}
\mathcal V(\theta+s\Delta\theta_n)^{2p_1} \le C\big(\mathcal V(\theta)^{2p_1}+\mathcal V(\theta+\Delta\theta_n)^{2p_1}\big),
\qquad s\in(0,1).
\end{align}
Hence,
\begin{align}
\mathrm{Var}(\psi(\theta_{n\delta_t})\mid\theta) \le C\,\mathbb{E}\Big[\|\Delta\theta_n\|^2 \big(\mathcal V(\theta)^{2p_1}+\mathcal V(\theta_{n\delta_t})^{2p_1}\big) \mid \theta \Big].
\end{align}

Taking conditional expectation with respect to $\mathcal B_n$ and using the tower property yields
\begin{align}
\mathbb{E}\!\left[ \mathrm{Var}(\psi(\theta_{n\delta_t})\mid\theta) \;\middle|\; \theta \right] \le C\,\mathbb{E}\Big[\|\Delta\theta_n\|^2 \big(\mathcal V(\theta)^{2p_1}+\mathcal V(\theta_{n\delta_t})^{2p_1}\big) \mid \theta \Big].
\end{align}

By Assumption~\ref{ass:lyap}(v), the increment satisfies the conditional second-moment bound
\begin{equation}\label{eq:inc-second-moment}
\mathbb{E}\big[\|\Delta\theta_n\|^2 \mid \theta_{(n-1)\delta_t}\big] \le C\,\delta_t\,\mathcal V(\theta_{(n-1)\delta_t})^{q},
\end{equation}
for some $q>0$. Taking total expectations and using Assumption~\ref{ass:lyap}(ii) to control moments of $\mathcal V(\theta_{n\delta_t})$ yields
\begin{equation}
\mathbb{E}\!\left[ \mathrm{Var}(\psi(\theta_{n\delta_t})\mid \theta_{(n-1)\delta_t}, \mathcal B_n) \right] \le C\,\delta_t.
\end{equation}
Substituting this bound into the definition of $\alpha_2$ gives
\begin{equation}
\alpha_2 \le \frac{1}{L^2\delta_t^2}\sum_{n=1}^L C\,\delta_t = \frac{C}{L\delta_t}.
\end{equation}
\item[\textbf{(iii)}] \textbf{$\alpha_3$.}
Set $X_n := \Delta V_n\psi(\theta_{(n-1)\delta_t})$. By the assumed unbiasedness we have $\mathbb E[X_n\mid \theta_{(n-1)\delta_t}]=0$, so $\{X_n\}$ is a martingale difference sequence and therefore cross-terms vanish:
\begin{align}
\mathbb E\Big[\Big(\sum_{n=1}^L X_n\Big)^2\Big] = \sum_{n=1}^L \mathbb E[X_n^2].
\end{align}
Moreover,
\begin{align}
|X_n| = |\zeta(\theta_{(n-1)\delta_t}; \mathcal B_n)\cdot\nabla\psi(\theta_{(n-1)\delta_t})| \le \|\zeta(\theta_{(n-1)\delta_t}; \mathcal B_n)\|\,\|\nabla\psi(\theta_{(n-1)\delta_t})\|,
\end{align}
so by Assumption~\ref{ass:lyap}(i-ii),
\begin{align}
\mathbb E[X_n^2] \le C\,\mathbb E\!\left[\|\zeta(\theta_{(n-1)\delta_t}; \mathcal B_n)\| ^2\right],
\end{align}
absorbing $\|\nabla\psi\|^2$ into the constant using the Lyapunov moment bounds. Hence
\begin{align}
\alpha_3 = \frac{1}{L^2}\sum_{n=1}^L \mathbb E[X_n^2] \le \frac{C}{L^2}\sum_{n=1}^L \mathbb E\!\left[\|\zeta(\theta_{(n-1)\delta_t}; \mathcal B_n)\|^2\right].
\end{align}
\item[\textbf{(iv)}] \textbf{$\alpha_4$.}
By Cauchy-Schwarz,
\begin{align}
\alpha_4 \le \frac{\delta_t^2}{L}\sum_{n=1}^L \mathbb E\!\left[(\Delta A_n\psi(\theta_{(n-1)\delta_t}))^2\right].
\end{align}
Using $\Delta A_n f = \frac12\,M_n:\nabla^2 f$ and the Hilbert-Schmidt inequality,
\begin{align}
|\Delta A_n\psi(\theta)| \le \frac12\,\|M_n(\theta)\|_F\,\|\nabla^2\psi(\theta)\|_F \le C\,\| M_n(\theta)\|\,\mathcal V(\theta)^{p_2},
\end{align}
where we used Assumption~\ref{ass:lyap}(i) with $k=2$. Taking expectations and using Assumptions~\ref{ass:lyap}(ii, iv),
we obtain
\begin{align}
\mathbb E\!\left[(\Delta A_n\psi(\theta_{(n-1)\delta_t}))^2\right] \le C\,\mathbb E\!\left[\|M_n(\theta_{(n-1)\delta_t}; \mathcal B_n)\|^2\right].
\end{align}
Therefore
\begin{align}
\alpha_4 \le \frac{C\,\delta_t^2}{L}\sum_{n=1}^L \mathbb E\!\left[\|M_n(\theta_{(n-1)\delta_t}; \mathcal B_n)\|^2\right].
\end{align}

\item[\textbf{(v)}] \textbf{$\alpha_5$.}
By Cauchy-Schwarz and the remainder bound from Lemma~\ref{lem:refined-weak-expansion},
\begin{align}
\alpha_5 \le \frac{\delta_t^2}{L}\sum_{n=1}^L \mathbb E\!\left[(\mathcal R_n\psi(\theta_{(n-1)\delta_t}))^2\right] \le \frac{C\,\delta_t^2}{L}\sum_{n=1}^L \mathbb E\!\left[\mathcal V(\theta_{(n-1)\delta_t})^{2p_0}\right] \le C\,\delta_t^2,
\end{align}
using Assumption~\ref{ass:lyap}(ii).
\end{enumerate}
Combining the bounds on $\alpha_1,\dots,\alpha_5$ yields \eqref{eq:refined-mse}.
\end{proof}

\subsection{Conditional Covariance and Moments of SGLD and SGLRW Updates}
\label{app:covariance-derivation}

We decompose the stochastic gradient as
\begin{align}
\widehat{\nabla U}(\theta; \mathcal B)
= \nabla U(\theta) + \zeta(\theta; \mathcal B), 
\qquad
\mathbb{E}_{\mathcal B}[\zeta(\theta; \mathcal B) \mid \theta] = 0,
\qquad
\mathrm{Cov}_{\mathcal B}[\zeta(\theta; \mathcal B)\mid \theta] = G(\theta),
\label{eq:app-stochgrad}
\end{align}
where $G(\theta)$ quantifies the minibatch-induced gradient covariance.

Recall the definition the second-order minibatch contribution $M_n$ used in
Theorem~\ref{thm:app_refined-mse}:
\begin{equation}
M_n(\theta,\mathcal B_n)
:= \delta_t^{-2}\,
\mathbb{E}_{\xi_n}\!\left[
\Delta\theta_n^{\mathrm{mb}}(\Delta\theta_n^{\mathrm{mb}})^\top
-
\Delta\theta_n^{\mathrm{fb}}(\Delta\theta_n^{\mathrm{fb}})^\top
\;\middle|\;
\theta_{(n-1)\delta_t}=\theta,\mathcal B_n
\right],
\end{equation}
where $\varepsilon_n$ denotes the internal randomness of the integrator.

\paragraph{SGLD.}
The standard SGLD update is $\Delta\theta_t = -\delta_t(\nabla U(\theta_t) + \zeta_t) + \sqrt{2\delta_t}\xi_t$ with $\xi_t \sim \mathcal{N}(0, I)$. Using the independence of $\xi_t$ and $\zeta_t$, and $\mathbb{E}[\zeta_t \mid \theta_t] = 0$, the first and second moments are:
\begin{align}
\mathbb{E}[\Delta\theta_t \mid \theta_t] &= -\delta_t\nabla U(\theta_t), \\
\mathbb{E}[\Delta\theta_t \Delta\theta_t^\top \mid \theta_t] &= 2\delta_t I + \delta_t^2\left( \nabla U(\theta_t)\nabla U(\theta_t)^\top + G(\theta_t) \right).
\end{align}

We next derive the second-order minibatch contribution $M_n$. Fix $\theta$ and a minibatch $\mathcal B$, define the full-batch and minibatch increments (sharing the same Gaussian $\xi$)
\begin{equation}
\Delta\theta^{\mathrm{fb}} = -\delta_t \nabla U(\theta) + \sqrt{2\delta_t}\,\xi,
\qquad
\Delta\theta^{\mathrm{mb}} = -\delta_t(\nabla U(\theta)+\zeta) + \sqrt{2\delta_t}\,\xi
= \Delta\theta^{\mathrm{fb}}-\delta_t\zeta.
\end{equation}
Expanding the outer products gives
\begin{align}
\Delta\theta^{\mathrm{mb}}\Delta\theta^{\mathrm{mb}\top}
-\Delta\theta^{\mathrm{fb}}\Delta\theta^{\mathrm{fb}\top}
&=
-\delta_t\,\Delta\theta^{\mathrm{fb}}\zeta^\top
-\delta_t\,\zeta(\Delta\theta^{\mathrm{fb}})^\top
+\delta_t^2\,\zeta\zeta^\top.
\end{align}
Taking conditional expectation over the internal randomness $\xi$ and using
$\mathbb E_\xi[\Delta\theta^{\mathrm{fb}}\mid\theta]=-\delta_t \nabla U(\theta)$ yields
\begin{equation}
\mathbb E_\xi\!\left[
\Delta\theta^{\mathrm{mb}}\Delta\theta^{\mathrm{mb}\top}
-\Delta\theta^{\mathrm{fb}}\Delta\theta^{\mathrm{fb}\top}
\ \middle|\ \theta,\mathcal B\right]
=
\delta_t^2\Big(\zeta\zeta^\top + \nabla U(\theta)\zeta^\top + \zeta \nabla U(\theta)^\top\Big).
\end{equation}
Therefore, with $M_n$ defined as in \eqref{eq:app_defAn},
\begin{equation}
M_{n,\mathrm{SGLD}}(\theta;\mathcal B)
=
\zeta(\theta;\mathcal B)\zeta(\theta;\mathcal B)^\top
+\nabla U(\theta)\,\zeta(\theta;\mathcal B)^\top
+\zeta(\theta;\mathcal B)\,\nabla U(\theta)^\top.
\label{eq:Mn-sgld}
\end{equation}
Averaging additionally over minibatches gives
$\mathbb E_{\mathcal B}[M_{n,\mathrm{SGLD}}(\theta,\mathcal B)\mid\theta]=G(\theta)$.

For the third moment, let $u_t = -\delta_t(\nabla U + \zeta_t)$ and $w_t = \sqrt{2\delta_t}\xi_t$. Expanding $\mathbb{E}[(u_t + w_t)^{\otimes 3} \mid \theta_t]$, terms with odd powers of $\xi_t$ vanish. The remaining terms are $\mathbb{E}[u_t^{\otimes 3}]$ and the cross-terms $\mathbb{E}[u_t \otimes w_t \otimes w_t]$ (and permutations).
The leading $O(\delta_t^2)$ error comes from the cross-terms:
\begin{align*}
\mathbb{E}[u_{t,i} w_{t,j} w_{t,k} \mid \theta_t] &= \mathbb{E}[-\delta_t(\partial_i U + \zeta_i) \cdot 2\delta_t \delta_{jk}] = -2\delta_t^2 \partial_i U \delta_{jk}.
\end{align*}
Because $\mathbb{E}[\zeta_i]=0$, the $O(\delta_t^2)$ part of the tensor is identical for full-batch and minibatch schemes. The third-order moments of the noise appear in the $\mathbb{E}[u_t^{\otimes 3}]$ expansion. By expanding the cubic terms and using $\mathbb{E}[\zeta_t]=0$, the tensor entries decompose as follows:

\begin{equation}
\label{eq:app-sgld-third-moment}
\mathbb E\!\left[\Delta\theta_i\,\Delta\theta_j\,\Delta\theta_k \mid \theta_t \right]
=
\begin{cases}
-6\delta_t^2\,\partial_i U
-\delta_t^3\Big(
(\partial_i U)^3
+3\,\partial_i U\,G_{ii}
+\mathbb E[\zeta_i^3\mid\theta_t]
\Big),
& i=j=k,\\[6pt]
-2\delta_t^2\,\partial_k U
-\delta_t^3\Big(
(\partial_i U)^2\,\partial_k U
+\partial_k U\,G_{ii}
+2\,\partial_i U\,G_{ik}
+\mathbb E[\zeta_i^2\zeta_k\mid\theta_t]
\Big),
& i=j\neq k,\\[6pt]
-\delta_t^3\Big(
\partial_i U\,\partial_j U\,\partial_k U
+\sum_{\text{cyc}(i,j,k)} \partial_i U G_{jk}
+\mathbb E[\zeta_i\zeta_j\zeta_k\mid\theta_t]
\Big),
& i,j,k\ \text{all distinct},
\end{cases}
\end{equation}
and permutations.

\paragraph{SGLRW.}
For the lattice random walk (LRW) discretisation, each coordinate $i$
takes a binary step $\Delta\theta_{t,i}\in\{\pm\sqrt{2\delta_t}\}$ with probabilities
\begin{align}
\mathbb{P}\!\left[\Delta\theta_{t,i}=\sqrt{2\delta_t}\,\middle|\,\theta_t,\zeta_t\right]
&= \tfrac{1}{2}
  - \tfrac{1}{2}\sqrt{\tfrac{\delta_t}{2}}\,[\partial_i U(\theta_t)+\zeta_{t,i}],\\
\mathbb{P}\!\left[\Delta\theta_{t,i}=-\sqrt{2\delta_t}\,\middle|\,\theta_t,\zeta_t\right]
&= \tfrac{1}{2}
  + \tfrac{1}{2}\sqrt{\tfrac{\delta_t}{2}}\,[\partial_i U(\theta_t)+\zeta_{t,i}].
\label{eq:app-sglrw-prob}
\end{align}
Conditionally on $(\theta_t,\zeta_t)$ the coordinates are independent.
A short calculation gives
\begin{align}
\mathbb{E}[\Delta\theta_t \mid \theta_t,\zeta_t] &= -\delta_t(\nabla U(\theta_t)+\zeta_t), 
\label{eq:app-sglrw-mean}\\
\mathbb{E}[\Delta\theta_t^2 \mid \theta_t,\zeta_t] &= 2\delta_tI. 
\end{align}
Averaging over $\zeta_t$ yields
$\mathbb{E}[\Delta\theta_t \mid \theta_t] = -\delta_t\nabla U(\theta_t)$
and $\mathbb{E}[\Delta\theta_{t,i}^2 \mid \theta_t] = 2\delta_t$.
For off–diagonal elements $(i\neq j)$,
\begin{align}
\mathbb{E}[\Delta\theta_{t,i}\Delta\theta_{t,j} \mid \theta_t]
 = \mathbb{E}[\mathbb{E}[\Delta\theta_{t,i} \mid \zeta_t]\mathbb{E}[\Delta\theta_{t,j} \mid \zeta_t] \mid \theta_t] = \delta_t^2\big[\partial_i U(\theta_t)\partial_j U(\theta_t) + G_{ij}(\theta_t)\big],
\label{eq:app-sglrw-cross}
\end{align}
hence we have
\begin{align}
\mathbb{E}[\Delta\theta_t \Delta\theta_t^\top \mid \theta_t]
=
2\delta_t I
+
\delta_t^2\,\mathrm{offdiag}\!\left(
\nabla U(\theta_t)\nabla U(\theta_t)^\top
+
G(\theta_t)
\right).
\label{eq:app-sglrw-second-moment}
\end{align}

We next compute the second-order minibatch contribution $M_n$.
Fix $\theta$ and a minibatch $\mathcal B$. For SGLRW, since $\Delta\theta_i^2\equiv 2\delta_t$ deterministically, the diagonal entries of
$\mathbb E_{\varepsilon}[\Delta\theta\Delta\theta^\top\mid\theta,\mathcal B]$ coincide with their full-batch counterparts, and therefore
\begin{equation}
\big(M_{n,\mathrm{SGLRW}}(\theta,\mathcal B)\big)_{ii}=0,\qquad i=1,\dots,d.
\label{eq:Mn-sglrw-diag}
\end{equation}
For $i\neq j$, conditional independence (given $(\theta,\mathcal B)$) implies
\begin{align}
\mathbb{E}_{\varepsilon}[\Delta\theta_i\Delta\theta_j \mid \theta,\mathcal B]
&=
\mathbb{E}_{\varepsilon}[\Delta\theta_i \mid \theta,\mathcal B]\,
\mathbb{E}_{\varepsilon}[\Delta\theta_j \mid \theta,\mathcal B] \nonumber\\
&=
\delta_t^2\big(\partial_i U(\theta)+\zeta_i\big)\big(\partial_j U(\theta)+\zeta_j\big),
\label{eq:app-sglrw-cross-cond}
\end{align}
where we used \eqref{eq:app-sglrw-mean}. The corresponding full-batch term is
$\mathbb{E}_{\varepsilon}[\Delta\theta_i^{\mathrm{fb}}\Delta\theta_j^{\mathrm{fb}}\mid\theta]
= \delta_t^2\,\partial_i U(\theta)\partial_j U(\theta)$ for $i\neq j$.
Subtracting and rescaling therefore gives, for $i\neq j$,
\begin{equation}
\big(M_{n,\mathrm{SGLRW}}(\theta;\mathcal B)\big)_{ij}
=
\partial_i U(\theta)\,\zeta_j
+\zeta_i\,\partial_j U(\theta)
+\zeta_i\zeta_j.
\label{eq:Mn-sglrw-offdiag-entry}
\end{equation}
Equivalently,
\begin{equation}
M_{n,\mathrm{SGLRW}}(\theta;\mathcal B)
=
\mathrm{offdiag}\!\left(
\zeta\zeta^\top + \nabla U(\theta)\zeta^\top + \zeta \nabla U(\theta)^\top
\right)
=
\mathrm{offdiag}\!\big(M_{n,\mathrm{SGLD}}(\theta;\mathcal B)\big).
\label{eq:Mn-sglrw}
\end{equation}
Averaging additionally over minibatches gives
$\mathbb E_{\mathcal B}[M_{n,\mathrm{SGLRW}}(\theta,\mathcal B)\mid\theta]=\mathrm{offdiag}\!\left(G(\theta)\right)$.

Since  $\Delta\theta_{t,i}\in\{\pm\sqrt{2\delta_t}\}$, we have the identity $\Delta\theta_{t,i}^3 = 2\delta_t\,\Delta\theta_{t,i}$.
Higher-order moments are computed via the law of iterated expectations. For the case where at least two indices are equal:
\begin{align}
\mathbb E[\Delta\theta_{t,i}^3 \mid \theta_t] &= 2\delta_t \mathbb E[\Delta\theta_{t,i}\mid \theta_t] = -2\delta_t^2\,\partial_i U(\theta_t),\\
\mathbb{E}[\Delta\theta_i^2 \Delta\theta_k \mid \theta_t] &= \mathbb{E}[2\delta_t \cdot \mathbb{E}[\Delta\theta_k \mid \theta_t, \zeta_t] \mid \theta_t] = \mathbb{E}[2\delta_t (-\delta_t(\partial_k U + \zeta_k))] = -2\delta_t^2 \partial_k U(\theta_t).
\end{align}
For fully distinct indices, the coordinates are coupled by the noise $\zeta_t$:
\begin{align}
\mathbb{E}[\Delta\theta_i \Delta\theta_j \Delta\theta_k \mid \theta_t] &= \mathbb{E}[ (-\delta_t(\partial_i U + \zeta_i))(-\delta_t(\partial_j U + \zeta_j))(-\delta_t(\partial_k U + \zeta_k)) \mid \theta_t] \nonumber \\
&= -\delta_t^3 \left( \partial_i U \partial_j U \partial_k U + \sum_{\text{cyc}(i,j,k)} \partial_i U G_{jk} + \mathbb{E}[\zeta_i \zeta_j \zeta_k \mid \theta_t] \right).
\end{align}
The resulting third-moment tensor entries are:
\begin{equation}
\mathbb E[\Delta\theta_{t,i}\Delta\theta_{t,j}\Delta\theta_{t,k}\mid \theta_t]
=
\begin{cases}
-2\delta_t^2\,\partial_i U(\theta_t), & i=j=k,\\
-2\delta_t^2\,\partial_k U(\theta_t), & i=j\neq k,\\
-\delta_t^3 \left( \partial_i U \partial_j U \partial_k U + \sum_{\text{cyc}(i,j,k)} \partial_i U G_{jk} + \mathbb{E}[\zeta_i \zeta_j \zeta_k \mid \theta_t] \right), & i,j,k \text{ all distinct},
\end{cases}
\label{eq:app-sglrw-third-moment}
\end{equation} 
and permutations.

\subsection{Full-Increment Covariance Analysis.}
\label{app:full-increment-clipping-limit}
We analyze \emph{full-increment clipping} for (SG)LD and show that it yields an incorrect diffusion limit.

We consider the SGLD increment
\begin{align}
\Delta\theta_t
= -\delta_t\,\widehat{\nabla U}(\theta_t)
  + \sqrt{2\delta_t}\,\xi_t,
\qquad \xi_t\sim\mathcal N(0,I),
\end{align}
and set the clipping radius $R_t=\sqrt{2\delta_t}$.
Introduce the rescaled increment
\begin{equation}
Z_t
:=\frac{\Delta\theta_t}{\sqrt{2\delta_t}}
=\xi_t-\sqrt{\frac{\delta_t}{2}}\;\widehat{\nabla U}(\theta_t).
\label{eq:Zt_def}
\end{equation}
Then we define the clipped increment by $\widetilde{\Delta\theta}_t := \sqrt{2\delta_t}\,\text{clip}(Z_t;1)$, where $\mathrm{clip}(x;R)_i = \mathrm{sign}(x_i)\min\{|x_i|,R\}$. Since $\text{clip}$ is bounded and Lipschitz, and since
$Z_t\to\xi$ in distribution (and in second moment) conditionally on $\theta_t$ whenever $\widehat{\nabla U}(\theta_t)$ has finite second moment, we obtain
\begin{align}
\lim_{\delta_t\to0}\frac{1}{\delta_t}\operatorname{Cov}\!\big[\widetilde{\Delta\theta}_t\mid\theta_t\big]
&=
\lim_{\delta_t\to0}2\,\operatorname{Cov}\!\big(\text{clip}(Z_t;1)\mid\theta_t\big)
=
2\,\operatorname{Cov}\!\big(\text{clip}(\xi;1)\big).
\label{eq:cov_limit_general}
\end{align}
Thus the diffusion limit is determined by the covariance of $\text{clip}(\xi;1)$. By independence across coordinates, this yields
\begin{equation}
\operatorname{Cov}\!\big(\text{clip}(\xi;1)\big)=sI,
\qquad
s=\mathbb E[\min(\xi_1^2,1)]
=1-\sqrt{\frac{2}{\pi}}e^{-1/2}\approx0.516,
\end{equation}
which can be shown by explicitly computing $\mathbb E[\min(\xi_1^2,1)]$. Therefore we have
\begin{equation}
\lim_{\delta_t\to0}\frac{1}{\delta_t}\operatorname{Cov}\!\big[\widetilde{\Delta\theta}_t\mid\theta_t\big]
=2s\,I.
\end{equation}
The limiting covariance equals $2sI$ with $s<1$, whereas the Langevin diffusion requires covariance $2I$. Hence full-increment clipping yields an incorrect diffusion limit and breaks convergence to the target distribution even as $\delta_t\to0$.

\subsection{Full Experimental Details: Sentiment Classification}
\label{app:sentiment-classification}

This appendix provides complete experimental details and full results for the sentiment classification experiments presented in Section~\ref{sec:LMM}. We report results for all evaluated learning-rate schedules, training-set sizes, and minibatch sizes, using fixed embeddings extracted from a pretrained OPT-350M model.

\paragraph{Experimental grid.}
For each method (clipped SGLD and SGLRW), we evaluate training set sizes $N \in \{10{,}000, 15{,}000, 20{,}000, 25{,}000\}$ and minibatch sizes $B \in \{8, 16, 32, 64, 128\}$. For each configuration, we perform three independent runs, each consisting of 15 independent chains of length 10,000, discarding the first 5,000 iterations as burn-in.

\paragraph{Learning-rate schedules.}
For the experiments we employ a decaying learning-rate schedule of the form
\begin{align}
\delta_t = s \cdot \eta_0 \, (t+1)^{-0.55},
\end{align}
where $\eta_0 = 5 \times 10^{-5}$ and $s \in \{0.5, 1.0, 3.0\}$ denotes a multiplicative scale factor. We refer to these as the conservative ($0.5\times$), baseline ($1.0\times$), and increased ($3.0\times$) learning-rate scales, respectively.

\paragraph{Metrics.}
Predictive metrics are computed on the held-out test set using posterior predictive probabilities obtained by averaging predicted probabilities across retained MCMC samples. Classification accuracy is computed by thresholding the averaged probability at $0.5$. The negative log-likelihood (NLL) is computed as
\begin{equation}
\mathrm{NLL}
= -\frac{1}{N} \sum_{i=1}^{N}
\bigl[
y_i \log(p_i) + (1 - y_i)\log(1 - p_i)
\bigr],
\end{equation}
where $y_i \in \{0,1\}$ is the true label for test example $i$, $N$ is the number of test examples, and $p_i$ denotes the averaged posterior predictive probability for sample $i$. Expected calibration error (ECE) is computed using $K=10$ equal-width bins over $[0,1]$,
\begin{equation}
\mathrm{ECE}
= \sum_{k=1}^{K} \frac{|B_k|}{N}
\left|
\mathrm{acc}(B_k) - \mathrm{conf}(B_k)
\right|,
\end{equation}
where $\mathrm{acc}(B_k)$ and $\mathrm{conf}(B_k)$ denote the empirical accuracy and mean predicted probability within bin $B_k$. Reported values correspond to means across chains; variability across chains is shown where indicated.

\paragraph{Reading the heatmaps.}
Figures~\ref{fig:acc_heatmaps_350m}, \ref{fig:nll_heatmaps_350m}, and~\ref{fig:ece_heatmaps_350m} report absolute predictive accuracy, negative log-likelihood (NLL), and expected calibration error (ECE) for clipped SGLD and SGLRW across the full experimental grid. These figures complement the relative-improvement summaries shown in the main text and allow inspection of absolute performance across regimes.
\label{app:sentiment-classification_results}
\begin{figure}
    \centering
    \includegraphics[width=\linewidth]{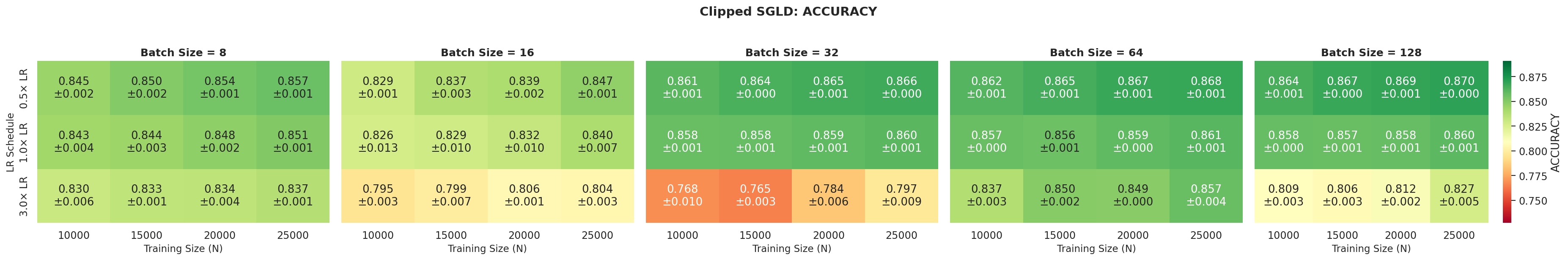}
    \includegraphics[width=\linewidth]{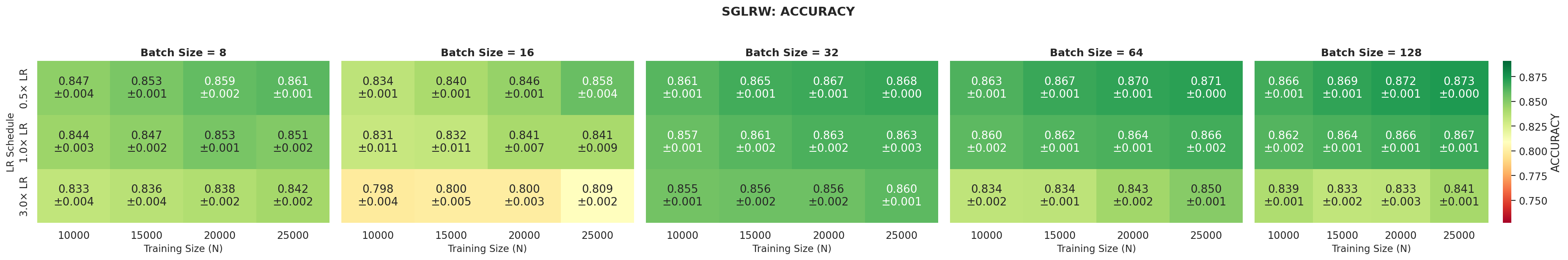}
    \caption{Predictive accuracy heatmaps for sentiment classification using fixed OPT-350M embeddings. 
    \textbf{Top:} Clipped SGLD. 
    \textbf{Bottom:} SGLRW.
    Each panel corresponds to a minibatch size $B$, with columns showing training-set size $N$ and rows indicating learning-rate scale. Values report mean test accuracy across chains, with standard deviation shown below each entry.}
    \label{fig:acc_heatmaps_350m}
\end{figure}

\begin{figure}
    \centering
    \includegraphics[width=\linewidth]{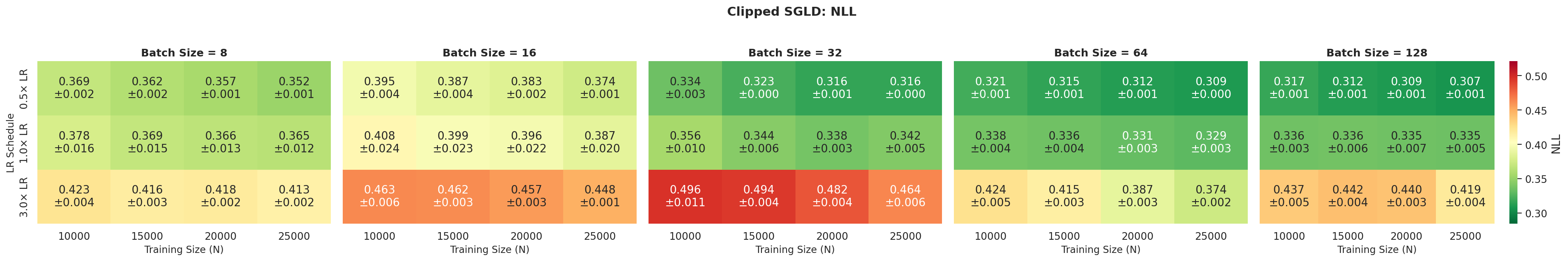}
    \includegraphics[width=\linewidth]{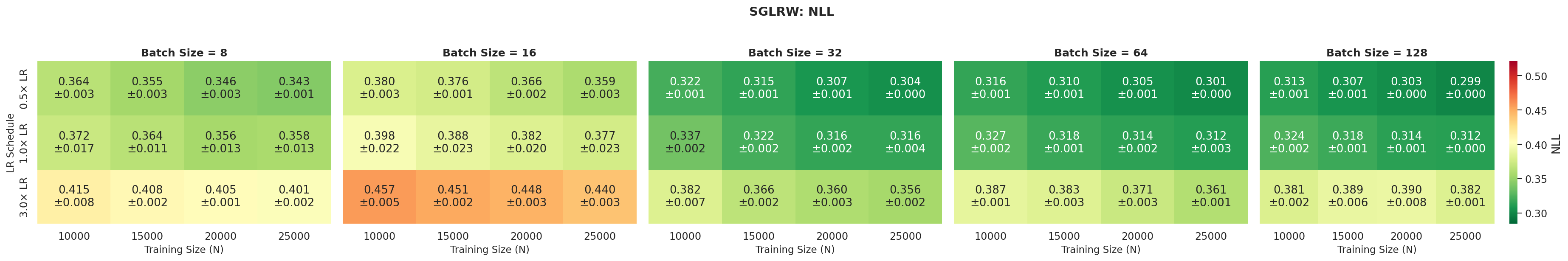}
    \caption{NLL heatmaps for sentiment classification using fixed OPT-350M embeddings. 
    \textbf{Top:} Clipped SGLD. 
    \textbf{Bottom:} SGLRW.
    Each panel corresponds to a minibatch size $B$, with columns showing training-set size $N$ and rows indicating learning-rate scale. Values report mean NLL values across chains, with standard deviation shown below each entry.}
    \label{fig:nll_heatmaps_350m}
\end{figure}

\begin{figure}
    \centering
    \includegraphics[width=\linewidth]{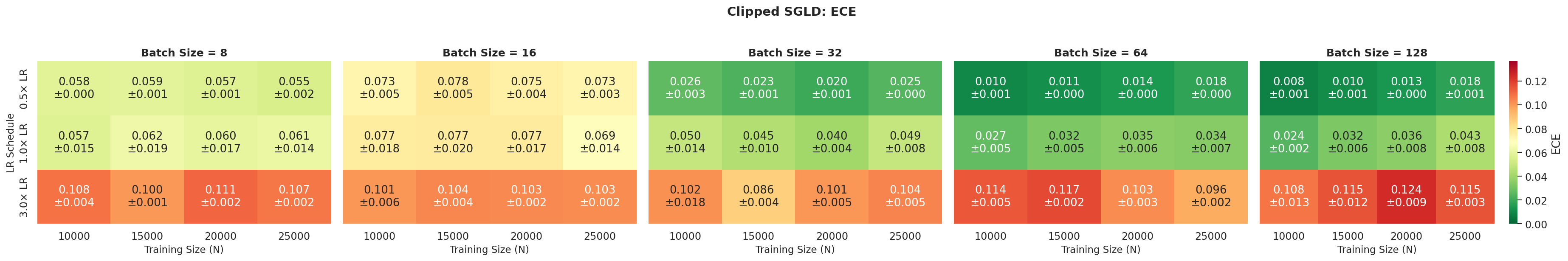}
    \includegraphics[width=\linewidth]{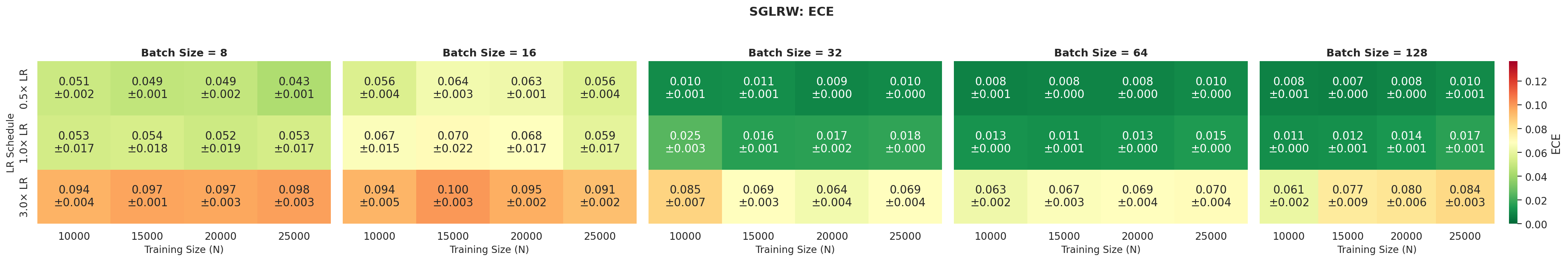}
    \caption{ECE heatmaps for sentiment classification using fixed OPT-350M embeddings. 
    \textbf{Top:} Clipped SGLD. 
    \textbf{Bottom:} SGLRW.
    Each panel corresponds to a minibatch size $B$, with columns showing training-set size $N$ and rows indicating learning-rate scale. Values report mean ECE values across chains, with standard deviation shown below each entry.}
    \label{fig:ece_heatmaps_350m}
\end{figure}

\end{document}